\documentclass{article}

\usepackage{microtype}
\usepackage{graphicx}
\usepackage{booktabs} %

\usepackage{url}

\usepackage{breakurl}
\usepackage[breaklinks]{hyperref}

\usepackage[accepted]{icml2024}

\usepackage{amsmath,amsfonts,bm}

\def\eqref#1{equation~\ref{#1}}

\def\1{\bm{1}}

\def\ve{{\bm{e}}}

\def\vh{{\bm{h}}}

\def\vr{{\bm{r}}}
\def\vs{{\bm{s}}}

\def\vv{{\bm{v}}}

\def\vx{{\bm{x}}}

\def\vz{{\bm{z}}}

\def\mA{{\bm{A}}}

\def\mD{{\bm{D}}}

\def\mP{{\bm{P}}}

\def\mV{{\bm{V}}}
\def\mW{{\bm{W}}}

\DeclareMathAlphabet{\mathsfit}{\encodingdefault}{\sfdefault}{m}{sl}
\SetMathAlphabet{\mathsfit}{bold}{\encodingdefault}{\sfdefault}{bx}{n}

\newcommand{\R}{\mathbb{R}}

\usepackage{hyperref}
\usepackage{url}

\usepackage{amsmath}
\usepackage{amssymb}
\usepackage{adjustbox}
\usepackage{amsthm}
\usepackage{bbm}
\usepackage{tikz}
\usetikzlibrary{shapes.geometric}
\graphicspath{ {./} }
\usepackage{subcaption}
\usepackage{booktabs}
\usepackage{enumitem}
\usepackage{placeins}
\usepackage{wrapfig}
\definecolor{lavender}{RGB}{230,230,250}
\definecolor{skyblue}{RGB}{135,206,250}

\newcommand{\N}{\mathbb{N}}

\newcommand{\add}[1]{#1}

\newtheorem{theorem}{Theorem}[section]
\newtheorem{lemma}[theorem]{Lemma}
\icmltitlerunning{Networked Inequality: Preferential Attachment Bias in Graph Neural Network Link Prediction}

\begin{document}

\twocolumn[
\icmltitle{Networked Inequality: Preferential Attachment Bias in \\
Graph Neural Network Link Prediction}

\icmlsetsymbol{equal}{*}

\begin{icmlauthorlist}
\icmlauthor{Arjun Subramonian}{ucla}
\icmlauthor{Levent Sagun}{meta}
\icmlauthor{Yizhou Sun}{ucla}
\end{icmlauthorlist}

\icmlaffiliation{ucla}{Computer Science Department, University of California, Los Angeles, USA}
\icmlaffiliation{meta}{Meta, Paris, France}

\icmlcorrespondingauthor{Arjun Subramonian}{arjunsub@cs.ucla.edu}

\icmlkeywords{graph learning, fairness, link prediction}

\vskip 0.3in
]

\printAffiliationsAndNotice{\icmlEqualContribution} %

\begin{abstract}
Graph neural network (GNN) link prediction is increasingly deployed in citation, collaboration, and online social networks to recommend academic literature, collaborators, and friends. While prior research has investigated the dyadic fairness of GNN link prediction, the within-group (e.g., queer women) fairness and ``rich get richer'' dynamics of link prediction remain underexplored. However, these aspects have significant consequences for degree and power imbalances in networks. In this paper, we shed light on how degree bias in networks affects Graph Convolutional Network (GCN) link prediction. In particular, we theoretically uncover that GCNs with a symmetric normalized graph filter have a within-group preferential attachment bias. We validate our theoretical analysis on real-world citation, collaboration, and online social networks. We further bridge GCN's preferential attachment bias with unfairness in link prediction and propose a new within-group fairness metric. This metric quantifies disparities in link prediction scores within social groups, towards combating the amplification of degree and power disparities. Finally, we propose a simple training-time strategy to alleviate within-group unfairness, and we show that it is effective on citation, social, and credit networks.
\end{abstract}

\section{Introduction}

\begin{figure}
\centering
\scalebox{0.6}{%
\begin{tikzpicture}[node distance={35mm}, thick, main/.style = {draw}, square/.style={regular polygon,regular polygon sides=4}] 
\node[main, circle, fill=lavender] (1) {$\textsc{CS}_1$}; 
\node[main, circle, fill=lavender] (2) [above right of=1] {$\textsc{CS}_2$};
\node[main, circle, fill=lavender] (3) [below right of=1] {$\textsc{CS}_3$}; 
\node[main, square, fill=lavender] (4) [above right of=3] {$\textsc{CS}_4$};
\node[main, circle, fill=lavender] (5) [above right of=4] {$\textsc{CS}_5$}; 
\node[main, draw=none] (7) [below of=4] {}; 
\node[main, draw=none] (8) [right of=4] {}; 
\node[main, circle, fill=skyblue] (6) [below right of=4] {$\textsc{Edu}_1$};
\draw[-] (1) -- (4);
\draw[-] (3) -- (4);
\draw[-] (2) -- (4);
\draw[-] (2) -- (5);
\draw[dashed] (4) -- (5);
\draw[dotted] (7) to [out=100,in=150,looseness=1] (8);
\draw[dashed] (4) -- (6);
\end{tikzpicture}} 
\caption{An academic collaboration network where nodes are Computer Science (\textsc{CS}) and Education (\textsc{Edu}) researchers, solid edges are current or past collaborations, and dashed edges are collaborations recommended by a GCN. Circular nodes are women and square nodes are men.
}
\label{fig:toy-collaboration-network}
\end{figure}
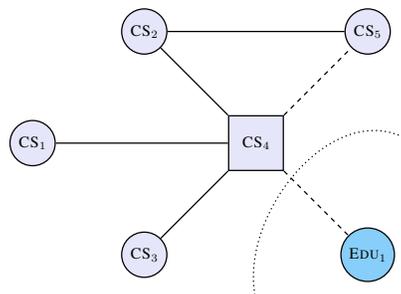

Link prediction (LP) using GNNs is increasingly leveraged to recommend friends in social networks \citep{Fan2019GraphNN, Sankar2021GraphNN}, as well as by scholarly tools to recommend academic literature in citation networks \citep{Xie2021CitationRecommendation}. In recent years, graph learning researchers have raised concerns about the unfairness of GNN LP \citep{li2021on, Current_2022, Li_Wang_Ning_Wang_2022}. This unfairness is often attributed to graph structure, including the stratification of social groups; for example, online networks are usually segregated by ethnicity \citep{Hofstra2017}. However, most fair GNN LP research has focused on dyadic fairness, i.e., satisfying some notion of parity between inter-group and intra-group link predictions. This formulation neglects: 1) LP dynamics within social groups \citep{Kasy2021Fairness}; and 2) the ``rich get richer'' effect, i.e., the prediction of links at a higher rate with high-degree nodes \citep{Barabasi1999Scaling}. In the context of friend recommendation systems, the ``rich get richer'' effect can increase the number of links formed with high-degree individuals, which boosts their influence on other individuals in the network, and thus their power \citep{Bashardoust2022ReducingAD}.

In this paper, we shed light on how degree bias in networks affects GCN LP \citep{kipf2017semisupervised}. We theoretically and empirically find that GCNs with a symmetric normalized graph filter have a within-group preferential attachment (PA) bias in LP. Specifically, GCNs often output LP scores that are approximately proportional to the geometric mean of the (within-group) degrees of the incident nodes when the nodes belong to the same social group. (We elaborate on PA and our motivation in \S\ref{sec:pa-motivation}.) We focus on GCNs with symmetric and random walk normalized graph filters because they are popular architectures for graph deep learning, and they provide us with a reasonable setting to develop a rigorous theory of PA bias in GNN LP while leveraging tools from spectral graph theory.

Our finding can have significant implications for the fairness of GCN LP. For example, consider links within the \textsc{CS} social group in the toy academic collaboration network in Figure \ref{fig:toy-collaboration-network}. Because men in \textsc{CS}, on average, have a higher within-group degree ($\text{deg} = 3$) than women in \textsc{CS} ($\text{deg} = 1.25$), due to gender discrimination, a collaboration recommender system that uses a GCN can suggest men as collaborators at a higher rate.
This has the detrimental effect of further concentrating research collaborations among men, thereby reducing the influence of women in \textsc{CS} and reinforcing their marginalization in the field \citep{Yamamoto2022Gender}. Furthermore, considering this marginalization in the context of \textsc{CS} is important, as such marginalization may be less severe or different in \textsc{Edu}.

Our contributions are as follows:
\begin{enumerate}[topsep=0pt,itemsep=0pt,parsep=0pt,partopsep=0pt,leftmargin=*]
    \item We theoretically uncover that GCNs with a symmetric normalized graph filter have a within-group PA bias in LP (\S\ref{sec:symmetric-theory}). We validate our theoretical analysis on diverse real-world network datasets (e.g., citation, collaboration, online social networks) of varying size (\S\ref{sec:validating-theory}). In doing so, we lay a foundation to study this previously-unexplored PA bias in the GNN setting.
    \item We bridge GCN's PA bias with unfairness in LP (\S\ref{sec:fairness-implications-theory}, \S\ref{sec:fairness-implications-exp}). We contribute a new within-group fairness metric for LP, which quantifies disparities in LP scores within social groups, towards combating the amplification of degree and power disparities. To our knowledge, we are the first to study the within-group fairness of GNNs.
    \item We propose a training-time strategy to alleviate within-group unfairness (\S\ref{sec:fairness-regularizer-intro}), and we assess its effectiveness on citation, online social, and credit networks (\S\ref{sec:fairness-regularizer}). Our experiments reveal that even for this new form of unfairness, simple regularization approaches can be successful.
\end{enumerate}

\section{Related Work}
\label{sec:related-work}

\paragraph{Degree Bias in GNNs}

Numerous papers have investigated how GNN performance is degraded for low-degree nodes on node representation learning and classification tasks \citep{tang2020degree, Liu2021Tail, kang2022rawlsgcn, Xu2023Grace, Shomer2023DegreeKG}. \citet{Liu_Nguyen_Fang_2023} present a generalized notion of degree bias that considers different multi-hop structures around nodes and propose a framework to address it; in contrast to prior work, which focuses on \textit{degree equal opportunity} (i.e., similar accuracy for nodes with the same degree), \citet{Liu_Nguyen_Fang_2023} also study \textit{degree statistical parity} (i.e., similar prediction rates of each class for nodes with the same degree). Beyond node classification, \citet{Wang2022LinkBias} find GNN LP performance disparities across nodes with different degrees: low-degree nodes often benefit from higher performance than high-degree nodes. In this paper, we find that GCNs have a PA bias in LP, and present a new fairness metric which quantifies disparities in GNN LP scores within social groups. We focus on \textit{group fairness} (i.e., parity between groups)
rather than \textit{individual fairness} (i.e., treating similar individuals similarly); this is because producing similar LP scores for similar-degree individuals does not prevent high-degree individuals from unfairly amassing links, and thus power (cf. Figure \ref{fig:toy-collaboration-network}). We further compare our work to prior degree bias works in \S\ref{sec:deg-bias-comments}.

\paragraph{Fair Link Prediction} Prior work has investigated the unfairness of GNN LP \citep{li2021on, Current_2022, Li_Wang_Ning_Wang_2022}, often attributing it to graph structure, (e.g., stratification of social groups). However, most of this research has focused on dyadic fairness, i.e., satisfying some notion of parity between inter-group and intra-group links. Like \citet{Wang2022LinkBias}, we examine how degree bias impacts GNN LP; however, rather than focus on performance disparities across nodes with different degrees, we study GCN's PA bias and LP score disparities across (sub)groups.

\paragraph{Within-Group Fairness} Much previous work has studied within-group fairness, i.e., fairness over social subgroups (e.g., Black women, Indigenous men) defined over multiple axes (e.g., race, gender) \citep{Kearns2017PreventingFG, Foulds2020Intersectional, ghosh2021characterizing, Wang2022Intersectionality}. The motivation of this work is that classifiers can be fair with respect to two social axes separately, but be unfair to subgroups defined over both these axes. While prior research has termed this phenomenon \textit{intersectional} unfairness, we opt for \textit{within-group} unfairness to distinguish it from the critical framework of Intersectionality \citep{Ovalle2023Factoring}. We study within-group fairness in the GNN setting. In particular, our theoretical and empirical findings reveal that GCN LP can further marginalize social subgroups; this relates to the ``complexity'' tenet of Intersectionality, which expresses that the marginalization faced by, e.g., Black women, is non-additive and distinct from the marginalization faced by Black men and white women \citep{collins2020intersectionality}.

\paragraph{Bias and Power in Networks} A wealth of literature outside fair graph learning has examined how network structure enables discrimination and disparities in capital \citep{Fish2019GapsAccess, Stoica2020Seeding, zhang2021chasm, Bashardoust2022ReducingAD}. \citet{boyd2014networked} describe how an individual's position in a social network affects their access to jobs and public health information, as well as how they are surveilled. \citet{stoica2018ceiling} observe that high-degree accounts on Instagram overwhelmingly belong to men and recommendation algorithms further boost these accounts; complementarily, the authors find that even a simple, random walk-based recommendation algorithm can amplify degree disparities between social groups in networks modeled by PA dynamics. Similarly, we investigate how GCN LP can amplify degree disparities in networks and further concentrate power among high-degree individuals.

\section{Preliminaries}
\label{sec:preliminaries}

We have a simple, undirected $n$-node graph ${\cal G} = ({\cal V}, {\cal E})$ with doubly-weighted self-loops. The nodes have features $\left( \vx_i \right)_{i \in {\cal V}}$, with each $\vx_i \in \R^d$. We denote the adjacency matrix of ${\cal G}$ as $\mA \in \{0, 1\}^{n \times n}$ and the degree matrix as $\mD = \text{diag} \left( \left( \sum_{j \in {\cal V}} \mA_{i j} \right)_{i \in {\cal V}} \right)$, with $\mD \in \N^{n \times n}$.
We consider two $L$-layer GCN encoders: (1) $\Phi_s: \R^{n \times d} \to \R^{n \times d'}$ \citep{kipf2017semisupervised}, which uses a symmetric normalized filter, and (2) $\Phi_r: \R^{n \times d} \to \R^{n \times d'}$, which uses a random walk normalized filter. $\Phi_s$ and $\Phi_r$ compute node representations as, $\forall i \in {\cal V}$:
\begin{align}
\Phi_s \left(\left( \vx_j \right)_{j \in {\cal V}}\right)_i &= \vs_i^{(L)}, \Phi_r \left(\left( \vx_j \right)_{j \in {\cal V}}\right)_i = \vr_i^{(L)}  \\
\forall l \in [L], \vs_i^{(l)} &= \sigma^{(l)} \left(\sum_{j \in \Gamma (i)} \frac{\mW_s^{(l)} \vs_j^{(l - 1)}}{\sqrt{\mD_{i i} \mD_{j j}}} \right), \\
\forall l \in [L], \vr_i^{(l)} &= \sigma^{(l)} \left(\sum_{j \in \Gamma (i)} \frac{\mW_r^{(l)} \vr_j^{(l - 1)}}{\mD_{i i}} \right),
\end{align}
where $\left( \vs_i^{(0)} \right)_{i \in {\cal V}} = \left( \vr_i^{(0)} \right)_{i \in {\cal V}} = \left( \vx_i \right)_{i \in {\cal V}}$; $\Gamma(i)$ is the 1-hop neighborhood of $i$; $\mW_s^{(l)}$ and $\mW_r^{(l)}$ are the weight matrices corresponding to layer $l$ of $\Phi_s$ and $\Phi_r$, respectively; for $l \in [L - 1], \sigma^{(l)}$ is a ReLU non-linearity; and $\sigma^{(L)}$ is the identity function.
We now consider the first-order Taylor expansions of $\Phi_s$ and $\Phi_r$ around $\left( \mathbf{0} \right)_{i \in {\cal V}}$:
\begin{align}
\vs_i^{(L)} = \sum_{j \in {\cal V}} \left[ \frac{\partial \vs_i^{(L)}}{\partial \vx_j} \right] \vx_j + \xi \left(\vs_i^{(L)}\right), \\
\vr_i^{(L)} = \sum_{j \in {\cal V}} \left[ \frac{\partial \vr_i^{(L)}}{\partial \vx_j} \right] \vx_j + \xi \left(\vr_i^{(L)}\right),\label{eqn:taylor}
\end{align}
where $\xi$ is the error of the first-order approximations. This error is low when $\left( \vx_i \right)_{i \in {\cal V}}$ are close to $\mathbf{0}$,
which we validate empirically in \S\ref{sec:validating-theory}.
Furthermore, we consider an inner-product LP score function $f_{LP}: \R^{d'} \times \R^{d'} \to \R$:
\begin{align}
\label{eqn:score-function}
f_{LP} \left(\vh_i^{(L)}, \vh_j^{(L)}\right) = \left( \vh_i^{(L)} \right)^\intercal \vh_j^{(L)},
\end{align}
where $\vh_i^{(L)}$ is the last-layer representation for node $i$.
\add{While it is common to use a vanilla GCN and inner-product score function for LP \citep{pyg-link-pred}, researchers have proposed methods to improve the expressivity of node representations for LP by capturing subgraph information \citep{Zhang2018LPGNN, Li2020DE, chamberlain2023graph}. Our theoretical findings remain relevant to methods that ultimately use a GCN to predict links (e.g., \citet{Zhang2018LPGNN, Li2020DE}), as we do not make assumptions about the features passed to the GCN (i.e., they could be distance encodings, SEAL node embeddings, etc.) Our results may also generalize to GNN architectures that use a degree-normalized graph filter, e.g., Graph Attention Networks \citep{veličković2018graph}. Studying the fairness of more expressive LP methods is an interesting direction for future research. Furthermore, although we only consider an inner-product LP score function in our theoretical analysis, we also run experiments with a Hadamard product and MLP score function (cf. \S\ref{sec:hadamard-exp}), and we find that our theoretical analysis is still relevant to and reasonably supports the experimental results.}

\section{Theoretical Analysis}
\label{sec:theoretical-analysis}

We leverage spectral graph theory to study how degree bias affects GCN LP. Theoretically, we find that GCNs with a symmetric normalized graph filter have a within-group PA bias (\S\ref{sec:symmetric-theory}), but GCNs with a random walk normalized filter may lack such a bias (\S\ref{sec:random-walk-theory}). We further bridge GCN's PA bias with unfairness in GCN LP, proposing a new LP within-group fairness metric (\S\ref{sec:fairness-implications-theory}) and a simple training-time strategy to alleviate unfairness (\S\ref{sec:fairness-regularizer-intro}). We empirically validate our theoretical results and fairness strategy in \S\ref{sec:experiments}. We provide proofs for all theoretical results in \S\ref{sec:proofs}.

Our ultimate goal is to bound the expected LP scores $\mathop{\mathbb{E}} \left[ f_{LP} \left( \vs_i^{(L)}, \vs_j^{(L)}\right) \right]$ and $\mathop{\mathbb{E}} \left[ f_{LP} \left( \vr_i^{(L)}, \vr_j^{(L)}\right) \right]$ for nodes $i, j$ in the same social group in terms of the degrees of $i, j$. We begin with Lemma \ref{lemma:taylor}, which expresses GCN representations (in expectation) as a linear combination of the initial node features. In doing so, we decouple the computation of GCN representations from the non-linearities $\sigma^{(l)}$.

\begin{lemma}
\label{lemma:taylor}
Similarly to \citet{Xu2018RepresentationLO}, assume that each path from node $i \to j$ in the computation graph of $\Phi_s$ is independently activated with probability $\rho_s (i)$, and similarly, $\rho_r (i)$ for $\Phi_r$ (cf. \S\ref{sec:taylor-lemma-comments}). Furthermore, suppose that $\mathop{\mathbb{E}} \left[ \xi \left(\vs_i^{(L)}\right) \right] = \mathop{\mathbb{E}} \left[ \xi \left(\vr_i^{(L)}\right) \right] = \mathbf{0}$, where the expectations are taken over the probability distributions of paths activating.
We define $\alpha_j = \left(\prod_{l = L}^1 \mW_s^{(l)} \right) \vx_j$, and $\beta_j = \left(\prod_{l = L}^1 \mW_r^{(l)} \right) \vx_j$.
Then, $\forall i \in {\cal V}$:
\begin{align}
&\mathop{\mathbb{E}} \left[ \vs_i^{(L)} \right] = \sum_{j \in {\cal V}} \rho_s (i) \left( \mD^{-\frac{1}{2}} \mA \mD^{-\frac{1}{2}} \right)^L_{i j} \alpha_j, \\
&\mathop{\mathbb{E}} \left[ \vr_i^{(L)} \right] = \sum_{j \in {\cal V}} \rho_r (i) \left( \mD^{-1} \mA \right)^L_{i j} \beta_j.
\end{align}
\end{lemma}

Lemma \ref{lemma:taylor} demonstrates that under certain assumptions (which we show to be reasonable in \S\ref{sec:validating-theory}), the expected GCN representations can be expressed as a linear combination of the node features that depends on a normalized version of the adjacency matrix.

We now introduce social groups in $\cal G$ into our analysis. Suppose that $\cal V$ can be partitioned into $B$ disjoint
social groups $\{ S^{(b)} \}_{b \in [B]}$, such that $\bigcup_{b \in [B]} S^{(b)} = {\cal V}$ and $\bigcap_{b \in [B]} S^{(b)} = \emptyset$. Furthermore, we define ${\cal G}^{(b)}$ as the induced connected subgraph of ${\cal G}$ formed from $S^{(b)}$.
(If a group comprises $C > 1$ connected components, it can be treated as $C$ separate groups.)  Let $\widehat{\mA}$ be a within-group adjacency matrix that contains links between nodes in the same group, i.e., $\widehat{\mA}$ contains the link $(i, j)$ if and only if for some group $S^{(b)}$, $i, j \in S^{(b)}$. Without loss of generality, we reorder the rows and columns of $\widehat{\mA}$ and $\mA$ such that $\widehat{\mA}$ is a block matrix.
Let $\widehat{\mD}$ be the degree matrix of $\widehat{\mA}$.

\subsection{Symmetric Normalized Filter}
\label{sec:symmetric-theory}

We first focus on analyzing $\Phi_s$. We introduce the notation $\mP = \mD^{-\frac{1}{2}} \mA \mD^{-\frac{1}{2}}$ for the symmetric normalized adjacency matrix.
We further define $\widehat{\mP} = \widehat{\mD}^{-\frac{1}{2}} \widehat{\mA} \widehat{\mD}^{-\frac{1}{2}}$, which has the form $\begin{bmatrix} \widehat{\mP}^{(1)} & & \mathbf{0} \\
& \ddots & \\
\mathbf{0} & & \widehat{\mP}^{(B)} \end{bmatrix}$. Each $\widehat{\mP}^{(b)}$ admits the orthonormal spectral decomposition $\widehat{\mP}^{(b)} = \sum_{k = 1}^{\left| S^{(b)} \right|} \lambda_k^{(b)} \vv_k^{(b)} \left( \vv_k^{(b)} \right)^\intercal$.
Let $\left( \lambda^{(b)}_k \right)_{1 \leq k \leq \left| S^{(b)} \right|}$ be the eigenvalues of $\widehat{\mP}^{(b)}$ sorted in non-increasing order; the eigenvalues fall in the range $(-1, 1]$. By the spectral properties of $\widehat{\mP}^{(b)}$, $\lambda^{(b)}_1 = 1$.
Following \citet{Lovsz2001RandomWO}, we denote the \textit{spectral gap} of $\widehat{\mP}^{(b)}$ as $\lambda^{(b)} = \max \left\{ \lambda^{(b)}_2, \left|\lambda^{(b)}_{\left| S^{(b)} \right|} \right| \right\} < 1$; $\lambda^{(b)}_2$ corresponds to the smallest non-zero eigenvalue of the symmetric normalized graph Laplacian.
Let $\mP = \widehat{\mP} + \Xi^{(0)}$. If $\cal G$ is highly modular or approximately disconnected, then $\Xi^{(0)} \approxeq \mathbf{0}$, albeit with positive and non-positive entries. Finally, we define the volume $\text{vol} \left({\cal G}^{(b)}\right) = \sum_{k \in S^{(b)}} \widehat{\mD}_{k k}$.

In Lemma \ref{lemma:oversmoothing-sym}, we present an inequality for the entries of $\mP^L$ in terms of the spectral properties of $\widehat{\mP}$. We can then combine this inequality with Lemma \ref{lemma:taylor} to bound $\mathop{\mathbb{E}} \left[ \vs_i^{(L)} \right]$, and subsequently $\mathop{\mathbb{E}} \left[ f_{LP} \left( \vs_i^{(L)}, \vs_j^{(L)}\right) \right]$.

\begin{lemma}
\label{lemma:oversmoothing-sym}
For $i, j \in S^{(b)}$:
\begin{align}
&\left| \mP^L_{i j} - \frac{ \sqrt{\widehat{\mD}_{i i} \widehat{\mD}_{j j}}}{\text{vol} \left({\cal G}^{(b)}\right)} \right| \\
&\leq \zeta_s = \left( \lambda^{(b)} \right)^L + \sum_{l = 1}^L {L \choose l} \left\| \Xi^{(0)} \right\|^l_{op} \left\| \widehat{\mP} \right\|^{L - l}_{op},
\end{align}
where $\|\cdot\|_{op}$ is the operator norm. And for $i \in S^{(b)}, j \notin S^{(b)}$, $\left| \mP^L_{i j} - 0 \right| \leq \sum_{l = 1}^L {L \choose l} \left\| \Xi^{(0)} \right\|^l_{op} \left\| \widehat{\mP} \right\|^{L - l}_{op} \leq \zeta_s$.
\end{lemma}
The proof of Lemma \ref{lemma:oversmoothing-sym} is similar to spectral proofs of random walk convergence. When $L$ is small (e.g., 2 for many GCNs \citep{kipf2017semisupervised}) and $ \left\| \Xi^{(0)} \right\|_{op} \approxeq 0$, $\sum_{l = 1}^L {L \choose l} \left\| \Xi^{(0)} \right\|^l_{op} \left\| \widehat{\mP} \right\|^{L - l}_{op} \approxeq 0$. Furthermore, with significant stratification between social groups \citep{Hofstra2017} and high expansion within groups \citep{Malliaros2011ExpansionPO, Leskovec2008CommunitySI}, $\lambda^{(b)} << 1$. In this case, $\zeta_s \approxeq 0$ and $\mP^L_{i j} \approxeq \frac{ \sqrt{\widehat{\mD}_{i i} \widehat{\mD}_{j j}}}{\text{vol} \left({\cal G}^{(b)}\right)}$ for $i, j \in S^{(b)}$. Combining Lemmas \ref{lemma:taylor} and \ref{lemma:oversmoothing-sym}, $\Phi_s$ can oversmooth the expected representations to $\mathop{\mathbb{E}} \left[ \vs_i^{(L)} \right] \approxeq \rho_s (i) \sqrt{\widehat{\mD}_{i i}} \cdot \sum_{j \in S^{(b)}} \frac{ \sqrt{\widehat{\mD}_{j j}}}{\text{vol} \left({\cal G}^{(b)}\right)} \alpha_j$ \citep{Keriven2022NotTL, Giovanni2022UnderstandingCO}.
We use this knowledge to bound $\mathop{\mathbb{E}} \left[ f_{LP} \left( \vs_i^{(L)}, \vs_j^{(L)}\right) \right]$ in terms of the degrees of $i, j$.

\begin{theorem}
\label{thm:sym}
Following a relaxed assumption from \citet{Xu2018RepresentationLO}, for nodes $i, j \in S^{(b)}$, we assume that $\rho_s (i) = \rho_s (j) = \overline{\rho}_s (b)$.
Then:
\begin{align}
\label{eqn:sym-formulas}
&\left| \mathop{\mathbb{E}} \left[ f_{LP} \left( \vs_i^{(L)}, \vs_j^{(L)}\right) \right] - C_0 \sqrt{\widehat{\mD}_{i i} \widehat{\mD}_{j j}} \right| \\
&\leq \zeta_s \overline{\rho}_s^2 (b) \left( \sqrt{\widehat{\mD}_{i i}} + \sqrt{\widehat{\mD}_{j j}} \right) C_1 C_2 + \zeta_s^2 \overline{\rho}_s^2 (b) C_2^2, \\
\text{where:}\\
&C_0 = \overline{\rho}_s^2 (b) C_1^2, \\
&C_1 = \left\| \sum_{k \in S^{(b)}} \frac{ \sqrt{\widehat{\mD}_{k k}}}{\text{vol} ({\cal G}^{(b)})} \alpha_k \right\|_2, \\
&C_2 = \sum_{k \in {\cal V}} \| \alpha_k \|_2.
\end{align}
\end{theorem}
In simpler terms, Theorem \ref{thm:sym}
states that with social stratification and expansion, the expected LP score $\mathop{\mathbb{E}} \left[ f_{LP} \left( \vs_i^{(L)}, \vs_j^{(L)}\right) \right] \propto \sqrt{\widehat{\mD}_{i i} \widehat{\mD}_{j j}}$ approximately when $i, j$ belong to the same social group. This is because, as explained before Theorem \ref{thm:sym}, $\zeta_s \approxeq 0$, so the RHS of the bound is $\approxeq 0$.
This demonstrates that in LP, GCNs with a symmetric normalized graph filter have a within-group PA bias. If $\Phi_s$ positively influences the formation of links over time, this PA bias can drive ``rich get richer'' dynamics within social groups  \citep{stoica2018ceiling}. As shown in Figure \ref{fig:toy-collaboration-network} and \S\ref{sec:fairness-implications-theory}, such ``rich get richer'' dynamics can engender group unfairness when nodes' degrees are statistically associated with their group membership (\S\ref{sec:fairness-implications-theory}). An association between node degree and group membership depends on group size and homophily; in particular, when a group has many nodes and intra-links (i.e., is homophilous), there may be more nodes with a high within-group degree. Beyond fairness, Theorem \ref{thm:sym} reveals that GCNs do not align with theories that \textit{social rank} influences link formation, i.e., the likelihood of a link forming between nodes is proportional to their degree \textit{difference} \citep{Gu2018RaRESR}.

\subsection{Within-Group Fairness}
\label{sec:fairness-implications-theory}

We further investigate the fairness implications of the PA bias of $\Phi_s$ in LP. We first introduce an additional set of social groups. Suppose that ${\cal V}$ can also be partitioned into $D$ disjoint social groups $\{ T^{(d)} \}_{d \in [D]}$; then, we can consider intersections of $\{ S^{(b)} \}_{b \in [B]}$ and $\{ T^{(d)} \}_{d \in [D]}$. For example, revisiting Figure \ref{fig:toy-collaboration-network}, $S$ may correspond to academic discipline (e.g., \textsc{CS}, \textsc{Edu}) and $T$ may correspond to gender (e.g., men, women).
For simplicity, we let $D = 2$. We measure the unfairness $\Delta^{(b)}: \R^{d'} \times \R^{d'} \to \R$ of LP for group $b$ as:
\begin{align}
\label{eqn:fairness-metric}
&\Delta^{(b)} \left( \vh_i^{(L)}, \vh_j^{(L)} \right) := \\
&\Biggl| \mathop{\mathbb{E}}_{i, j \sim U((S^{(b)} \cap T^{(1)}) \times S^{(b)})} f_{LP} \left( \vh_i^{(L)}, \vh_j^{(L)} \right) \\
&- \mathop{\mathbb{E}}_{i, j \sim U((S^{(b)} \cap T^{(2)}) \times S^{(b)})} f_{LP} \left( \vh_i^{(L)}, \vh_j^{(L)} \right) \biggr|,
\end{align}
where $U(\cdot)$ is a discrete uniform distribution over the input set. $\Delta^{(b)}$ quantifies disparities in GCN LP scores within $S^{(b)}$ (with respect to $T^{(1)}$ and $T^{(2)}$). In other words, $\Delta^{(b)}$ measures differences in how GCNs allocate LP scores across subgroups, i.e., are links with nodes in one subgroup predicted at a higher rate than links with nodes in the other subgroup? Our metric is motivated by how GNN link predictions influence real-world link formation (e.g., GNN-based recommender systems use LP scores to rank suggested social connections), which has consequences for degree and power disparities.
Based on Theorem \ref{thm:sym}
and \S\ref{sec:approx-delta-sym-proof}, when $\zeta_s \approxeq 0$, we can estimate $\Delta^{(b)} \left( \vs_i^{(L)}, \vs_j^{(L)} \right)$ as:
\begin{align}
&\widehat{\Delta}^{(b)} \left( \vs_i^{(L)}, \vs_j^{(L)} \right) \\
&= \frac{\overline{\rho}_s^2 (b)}{\left| S^{(b)} \right|} \left\| \sum_{k \in S^{(b)}} \frac{ \sqrt{\widehat{\mD}_{k k}}}{\text{vol} ({\cal G}^{(b)})} \alpha_k \right\|^2_2 \Biggl| \sum_{j \in S^{(b)}} \sqrt{\widehat{\mD}_{j j}} \times \\
&\underbrace{\left( \mathop{\mathbb{E}}_{i \sim U(S^{(b)} \cap T^{(1)})}  \sqrt{\widehat{\mD}_{i i}} - \mathop{\mathbb{E}}_{i \sim U(S^{(b)} \cap T^{(2)})} \sqrt{\widehat{\mD}_{i i}} \right) }_{\text{degree disparity}} \biggr|
\label{eqn:approx-fairness-metric-sym}
\end{align}

This suggests that a large disparity in the degree of nodes in $S^{(b)} \cap T^{(1)}$ vs. $S^{(b)} \cap T^{(2)}$ can greatly increase the unfairness $\Delta^{(b)}$ of $\Phi_s$ LP. For example, in Figure \ref{fig:toy-collaboration-network}, the large degree disparity within \textsc{CS} (between men and women) entails that a GCN collaboration recommender system applied to the network will have a large $\Delta^{(b)}$. We empirically validate these fairness implications on diverse network datasets in \S\ref{sec:fairness-implications-exp}. While we consider pre-activation LP scores in Eqn. \ref{eqn:fairness-metric} (in line with prior work, e.g., \citet{li2021on}), we consider post-sigmoid scores $\sigma \left( f_{LP} \left( \vh_i^{(L)}, \vh_j^{(L)}\right) \right)$ (where $\sigma$ is the sigmoid function) in \S\ref{sec:fairness-implications-exp} and \S\ref{sec:fairness-regularizer}, as this simulates how LP scores may be processed in practice.

Ultimately, within-group unfairness is characteristic of all GNN link prediction methods that: (1) predict scores for links with magnitudes that are positively associated with the degrees of their incident nodes, and (2) are applied to graphs where within-group membership is associated with node degree.

\subsection{Random Walk Normalized Filter}
\label{sec:random-walk-theory}

We now follow similar steps as with $\Phi_s$ to understand how degree bias affects LP scores for $\Phi_r$. We redefine $\mP = \mD^{-1} \mA$, $\widehat{\mP} = \widehat{\mD}^{-1} \widehat{\mA}$, and the remaining notation from \S\ref{sec:symmetric-theory} accordingly for the random walk setting.

\begin{theorem}
\label{thm:rw}
Let $\zeta_r = \max_{u, v \in {\cal V}} \sqrt{\frac{\widehat{\mD}_{v v}}{\widehat{\mD}_{u u}}} \left( \lambda^{(b)} \right)^L + \sum_{l = 1}^L {L \choose l} \left\| \Xi^{(0)} \right\|^l_{op} \left\| \widehat{\mP} \right\|^{L - l}_{op}$. Furthermore, for nodes $i, j \in S^{(b)}$, assume that $\rho_r (i) = \rho_r (j) = \overline{\rho}_r (b)$. Combining Lemmas \ref{lemma:taylor} and \ref{lemma:oversmoothing-rw}:
\begin{align}
\label{eqn:rw-formulas}
&\left| \mathop{\mathbb{E}} \left[ f_{LP} \left( \vr_i^{(L)}, \vr_j^{(L)}\right) \right] - C_0 \right| \\
&\leq \zeta_r \overline{\rho}_r^2 (b) C_1 C_2 + \zeta_r^2 \overline{\rho}_r^2 (b) C_2^2, \\
\text{where:}\\
&C_0 = \overline{\rho}_r^2 (b) C_1^2, \\
&C_1 = \left\| \sum_{k \in S^{(b)}} \frac{ \widehat{\mD}_{k k}}{\text{vol} ({\cal G}^{(b)})} \beta_k \right\|, \\
&C_2 = \sum_{k \in {\cal V}} \| \beta_k \|_2.
\end{align}
\end{theorem}

In other words, if $\zeta_r \approxeq 0$, $\mathop{\mathbb{E}} \left[ f_{LP} \left( \vr_i^{(L)}, \vr_j^{(L)}\right) \right]$ is approximately constant when $i, j$ belong to the same social group. Based on Theorem \ref{thm:rw} and \S\ref{sec:approx-delta-rw-proof}, we can estimate $\Delta^{(b)} \left( \vs_i^{(L)}, \vs_j^{(L)} \right)$ as $\widehat{\Delta}^{(b)} \left( \vs_i^{(L)}, \vs_j^{(L)} \right) = 0$. Theoretically, this would suggest that a large disparity in the degree of nodes in $S^{(b)} \cap T^{(1)}$ vs. $S^{(b)} \cap T^{(2)}$ does not increase the unfairness $\Delta^{(b)}$ of $\Phi_r$ LP. However, we find empirically that this is not the case (\S\ref{sec:validating-theory}). Even so, we include theoretical results for the random walk filter to be more comprehensive with respect to filter choice, as well as be upfront about the limitations of our analysis in this case. We also seek to provide an example of how to apply our analysis to other filters, for researchers who would like to build on it in the future. For example, findings for the random walk filter could be relevant to the GAT filter \citep{veličković2018graph}, which is also a row-stochastic matrix.

In summary, in \S\ref{sec:theoretical-analysis}, we build on prior analysis techniques for random walks and GNNs. At a high level, we: (1) simplify the GCN architecture to be a linear function by truncating its Taylor expansion and considering node representations in expectation; (2) analyze the convergence of node representations via a spectral analysis of the convergence of short random walks within subgraphs (corresponding to social groups); and (3) use norm inequalities to estimate link prediction scores. Our analysis comprises numerous novel elements including:
\begin{enumerate}[topsep=0pt,itemsep=0pt,parsep=0pt,partopsep=0pt,leftmargin=*]
\item Analyzing the convergence of random walks within subgraphs, which requires accounting for the rate at which probability mass escapes from the subgraphs. In contrast, random walk results in the literature usually concern the convergence of random walks over an entire graph.
\item Uncovering properties of short random walks on graphs, since most GNNs are shallow. In contrast, random walk results in the literature often concern the stationary distribution of random walks.
\item Concretely relating theoretical properties of random walks to the fairness of GCN link prediction.
\end{enumerate}

\section{Fairness Regularizer}
\label{sec:fairness-regularizer-intro}

We propose a simple training-time solution to alleviate within-group LP unfairness regardless of graph filter type and GNN architecture. In particular, we can add a fairness regularization term ${\cal L}_{\text{fair}}$ to our original GNN training loss \citep{Kamishima2021Fairness}:
\begin{align}
{\cal L}_{\text{new}} = {\cal L}_{\text{orig}} + \lambda_{\text{fair}} {\cal L}_{\text{fair}} = {\cal L}_{\text{orig}} + \frac{\lambda_{\text{fair}}}{B} \sum_{b \in [B]} \Delta^{(b)},
\end{align}
where $\lambda_{\text{fair}}$ is a tunable hyperparameter that for higher values, pushes the GNN to learn fairer parameters. With our fairness strategy, we empirically observe a significant decrease in the average unfairness across groups $\frac{1}{B} \sum_{b \in [B]} \Delta^{(b)}$ without a severe drop in LP performance for GCN (\S\ref{sec:fairness-regularizer}).

\section{Experiments}
\label{sec:experiments}
In this section, we empirically validate our theoretical analysis (\S\ref{sec:validating-theory}) and the within-group fairness implications of GCN's LP PA bias (\S\ref{sec:fairness-implications-exp}) on diverse real-world network datasets of varying size. We further find that our simple training-time strategy to alleviate unfairness is effective on citation, online social, and credit networks (\S\ref{sec:fairness-regularizer}).
We release our code and data in our GitHub repository\footnote{\url{https://github.com/ArjunSubramonian/link_bias_amplification}}.
We present experimental results with 4-layer GCN encoders and a Hadamard product with MLP LP score function in \S\ref{sec:additional-experiments}, with similar conclusions.

\subsection{Validating Theoretical Analysis}
\label{sec:validating-theory}

We validate our theoretical analysis on 10 real-world network datasets (e.g., citation, collaboration, online social networks), which we describe in \S\ref{sec:validating-datasets}. Each dataset is natively intended for node classification; however, we adapt the datasets for LP, treating the connected components within the node classes as the social groups $S^{(b)}$. This design choice is reasonable, as in all the datasets, the classes naturally correspond to socially-relevant groupings of the nodes, or proxies thereof (e.g., in the LastFMAsia dataset, the classes are the home countries of users). Because we adopt the class labels for each dataset as the social group labels, the social groups are largely homophilic; this aligns with our assumptions when interpreting Theorems \ref{thm:sym} and \ref{thm:rw} that social groups are stratified in networks.

We train GCN encoders $\Phi_s$ and $\Phi_r$ for LP over 10 random seeds (cf. \S\ref{sec:models} for more details). In Figure \ref{fig:cora-cs-lastfmasia}, we plot the theoretic\footnote{While our theoretic scores resulted from our theoretical analysis in \S\ref{sec:theoretical-analysis}, we reiterate that our results in \S\ref{sec:theoretical-analysis} rely on the assumptions that we state and the theoretic score is not a ground-truth value.} LP score that we derive in \S\ref{sec:theoretical-analysis} against the GCN LP score \textit{for pairs of test nodes belonging to the same social group} (including positive and negative links). In particular, for $\Phi_s$, the theoretic LP score is $\overline{\rho}_s^2 (b) \sqrt{\widehat{\mD}_{i i} \widehat{\mD}_{j j}} \left\| \sum_{k \in S^{(b)}} \frac{ \sqrt{\widehat{\mD}_{k k}}}{\text{vol} ({\cal G}^{(b)})} \alpha_k \right\|^2_2$ and the GCN LP score is $f_{LP} \left( \vs_i^{(L)}, \vs_j^{(L)}\right)$ (cf. Theorem \ref{thm:sym}).
In contrast, for $\Phi_r$, the theoretic LP score is $\overline{\rho}_s^2 (b) \left\| \sum_{k \in S^{(b)}} \frac{ \widehat{\mD}_{k k}}{\text{vol} ({\cal G}^{(b)})} \beta_k \right\|^2_2$ and the GCN LP score is $f_{LP} \left( \vr_i^{(L)}, \vr_j^{(L)}\right)$ (cf. Theorem \ref{thm:rw}). For all the datasets, we estimate $\overline{\rho}_s^2 (b)$ and $\overline{\rho}_r^2 (b)$ separately for each social group $S^{(b)}$ as the slope of the least-squares regression line (through the data from $S^{(b)}$) that predicts the GCN score as a function of the theoretic score. Hence, we do not plot any pair of test nodes that is the only pair in $S^{(b)}$, as it is not possible to estimate $\overline{\rho}_s^2 (b)$. Further, the test AUC is consistently high, indicating that the GCNs are well-trained. The large range of each color in the plots indicates a diversity of LP scores within each social group.

\begin{figure*}[t!]
    \begin{subfigure}[t]{0.33\textwidth}
        \centering
        \includegraphics[width=\textwidth]{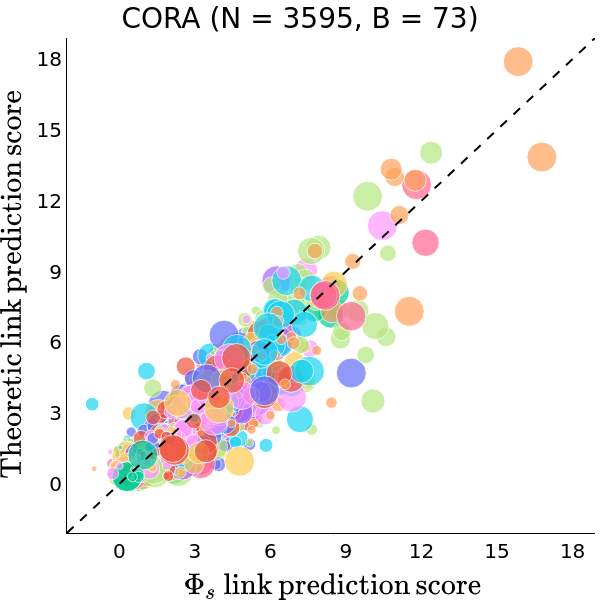}
    \end{subfigure}%
    \begin{subfigure}[t]{0.33\textwidth}
        \centering
        \includegraphics[width=\textwidth]{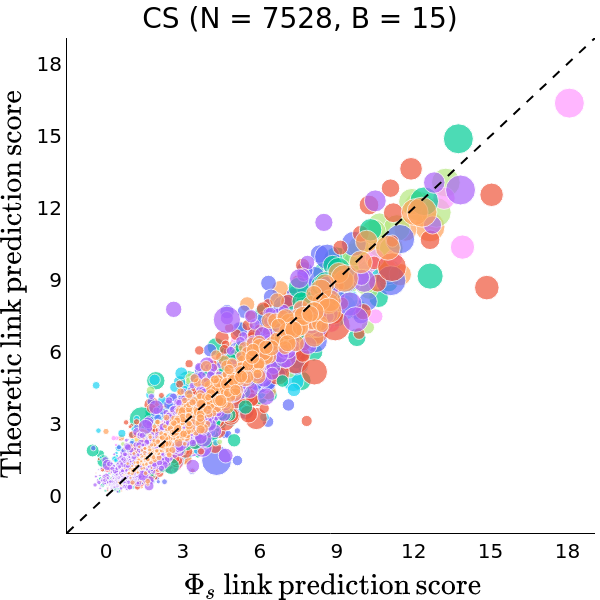}
    \end{subfigure}%
    \begin{subfigure}[t]{0.33\textwidth}
        \centering
        \includegraphics[width=\textwidth]{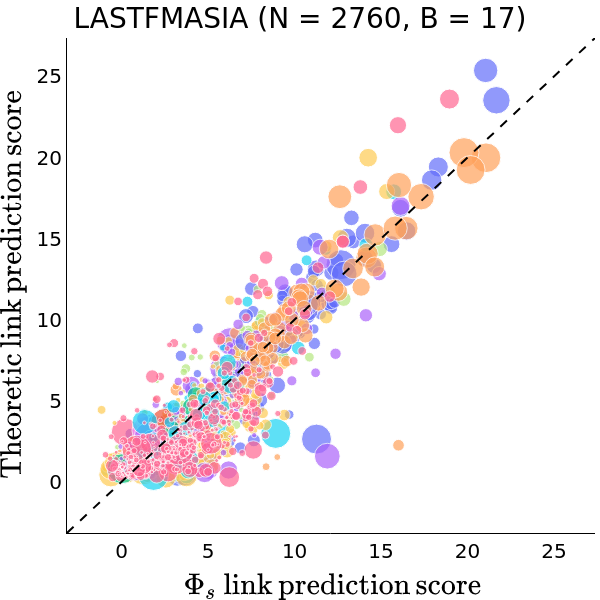}
    \end{subfigure}
    \label{fig:sym-cora-cs-lastfmasia}
    \centering
    \begin{subfigure}[t]{0.33\textwidth}
        \centering
        \includegraphics[width=\textwidth]{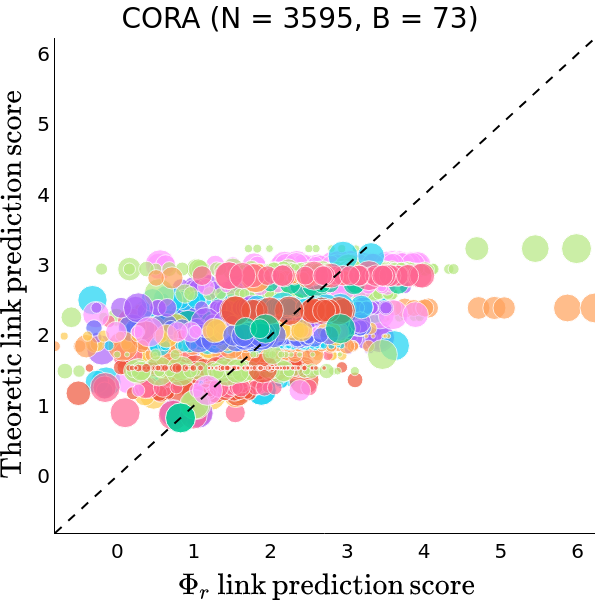}
    \end{subfigure}%
    \begin{subfigure}[t]{0.33\textwidth}
        \centering
        \includegraphics[width=\textwidth]{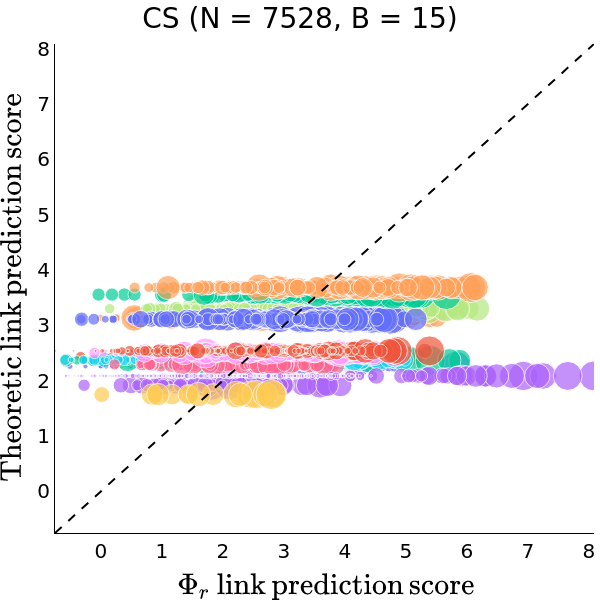}
    \end{subfigure}%
    \begin{subfigure}[t]{0.33\textwidth}
        \centering
        \includegraphics[width=\textwidth]{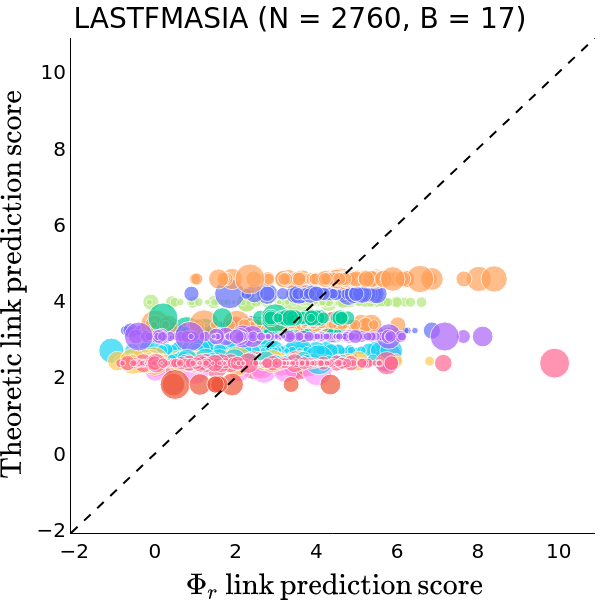}
    \end{subfigure}
    \par\bigskip
    \centering
    \begin{adjustbox}{max width=0.48\textwidth}
    \begin{tabular}{lrrr}
    \toprule
                & \textbf{NRMSE} ($\downarrow$) & \textbf{PCC} ($\uparrow$) & $\Phi_s$ \textbf{Test AUC} ($\uparrow$) \\
    \midrule
           CORA &             $0.038 \pm 0.006$ &         $0.884 \pm 0.008$ &              $0.927 \pm 0.008$ \\
       CITESEER &             $0.080 \pm 0.005$ &         $0.806 \pm 0.007$ &              $0.943 \pm 0.007$ \\
           DBLP &             $0.026 \pm 0.002$ &         $0.820 \pm 0.014$ &              $0.948 \pm 0.001$ \\
         PUBMED &             $0.061 \pm 0.008$ &         $0.774 \pm 0.018$ &              $0.927 \pm 0.010$ \\
             CS &             $0.036 \pm 0.006$ &         $0.917 \pm 0.019$ &              $0.932 \pm 0.008$ \\
        PHYSICS &             $0.042 \pm 0.003$ &         $0.822 \pm 0.021$ &              $0.946 \pm 0.003$ \\
     LASTFMASIA &             $0.064 \pm 0.003$ &         $0.889 \pm 0.004$ &              $0.962 \pm 0.001$ \\
             DE &             $0.025 \pm 0.003$ &         $0.795 \pm 0.043$ &              $0.913 \pm 0.003$ \\
             EN &             $0.041 \pm 0.002$ &         $0.542 \pm 0.013$ &              $0.876 \pm 0.003$ \\
             FR &             $0.030 \pm 0.002$ &         $0.743 \pm 0.026$ &              $0.910 \pm 0.005$ \\
    \bottomrule
    \end{tabular}
    \end{adjustbox}
    \quad
    \begin{adjustbox}{max width=0.48\textwidth}
    \begin{tabular}{lrrr}
    \toprule
                & \textbf{NRMSE} ($\downarrow$) & \textbf{PCC} ($\uparrow$) & $\Phi_r$ \textbf{Test AUC} ($\uparrow$) \\
    \midrule
           CORA & $0.101 \pm 0.029$ & $0.553 \pm 0.024$ & $0.942 \pm 0.005$ \\
            CITESEER & $0.170 \pm 0.016$ & $0.363 \pm 0.028$ & $0.934 \pm 0.003$ \\
            DBLP & $0.157 \pm 0.012$ & $0.235 \pm 0.022$ & $0.942 \pm 0.002$ \\
            PUBMED & $0.155 \pm 0.013$ & $0.079 \pm 0.029$ & $0.896 \pm 0.011$ \\
            CS & $0.101 \pm 0.027$ & $0.447 \pm 0.070$ & $0.939 \pm 0.003$ \\
            PHYSICS & $0.107 \pm 0.027$ & $0.264 \pm 0.038$ & $0.951 \pm 0.004$ \\
            LASTFMASIA & $0.123 \pm 0.016$ & $0.409 \pm 0.017$ & $0.949 \pm 0.001$ \\
            DE & $0.024 \pm 0.004$ & $0.074 \pm 0.016$ & $0.862 \pm 0.003$ \\
            EN & $0.065 \pm 0.006$ & $0.012 \pm 0.005$ & $0.850 \pm 0.002$ \\
            FR & $0.028 \pm 0.006$ & $0.006 \pm 0.003$ & $0.865 \pm 0.004$ \\
    \bottomrule
    \end{tabular}
    \end{adjustbox}
    \caption{The plots display the theoretic vs. GCN LP scores for the Cora, CS, and LastFMAsia datasets over 10 random seeds. (We include the plots for the remaining datasets in \S\ref{sec:remaining-plots}.) The \textbf{top row} of plots corresponds to $\Phi_s$, the \textbf{bottom row} to $\Phi_r$. In the plots, each circle corresponds to a single pair of test nodes (between which we are predicting a link). The center of each circle represents the mean of the theoretic and GCN scores and its area captures the range of scores. The color of each circle indicates the social group to which the node pair belongs. The plots include: (1) the total number of test node pairs $N$; (2) the number of social groups $B$; (3) the dashed line of equality for easy comparison of the theoretic and GCN scores. For all the datasets, the tables display: (1) the mean/standard deviation of the GCN test AUC on LP; and (2) the mean/standard deviation of the range-normalized\footnotemark root-mean-square deviation (NRMSE) \citep{otto2019rmse} and Pearson correlation coefficient (PCC) \citep{freedman2007statistics} of the theoretic LP scores as predictors of the GCN scores. The \textbf{left} table corresponds to $\Phi_s$, the \textbf{right} to $\Phi_r$.}
    \label{fig:cora-cs-lastfmasia}
\end{figure*}
\footnotetext{Normalized by the sample range of the GCN LP scores. Values fall between 0 and 1.}

We visually observe that the theoretic LP scores are strong predictors of the $\Phi_s$ scores for each dataset, validating our theoretical analysis. This strength is further confirmed by the generally low NRMSE and high PCC (except for the EN dataset). However, we observe a few cases in which our theoretical analysis does not line up with our experiments:

\begin{figure*}[ht!]
    \centering
    \begin{subfigure}[t]{0.33\textwidth}
        \centering
        \includegraphics[width=\textwidth]{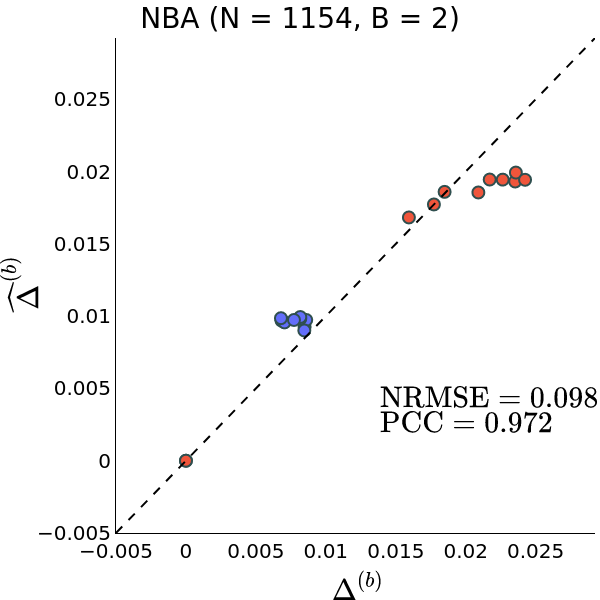}
    \end{subfigure}%
    \begin{subfigure}[t]{0.33\textwidth}
        \centering
        \includegraphics[width=\textwidth]{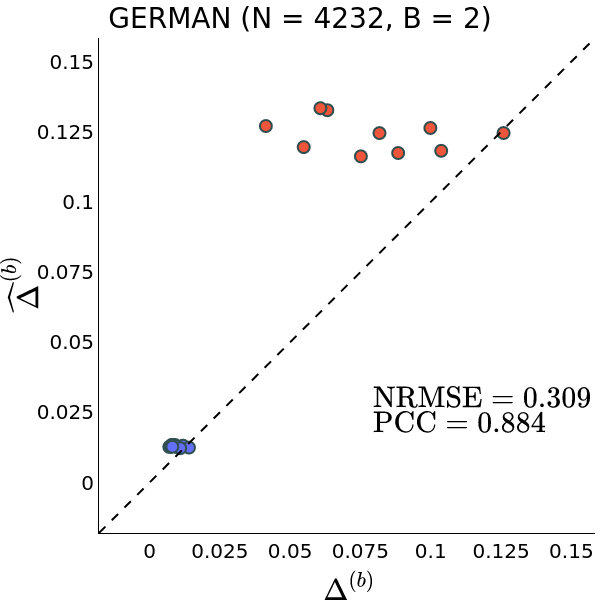}
    \end{subfigure}%
    \begin{subfigure}[t]{0.33\textwidth}
        \centering
        \includegraphics[width=\textwidth]{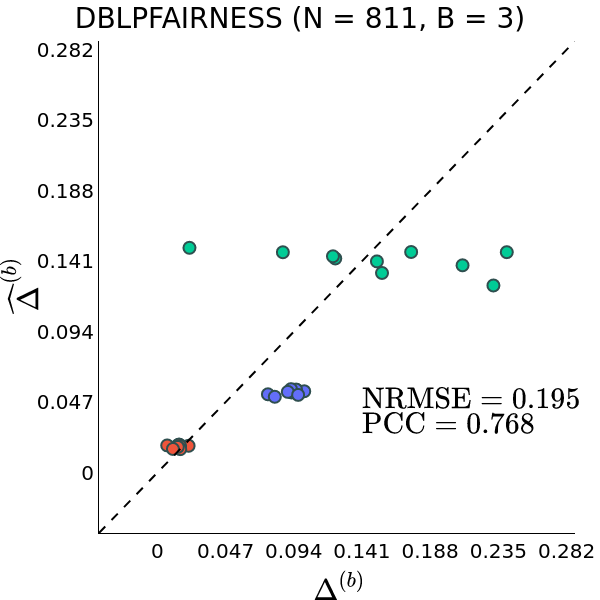}
    \end{subfigure}%
    \caption{The plots display $\widehat{\Delta}^{(b)}$ vs. $\Delta^{(b)}$ for $\Phi_s$ for the NBA, German, and DBLP-Fairness datasets over all $b \in [B]$ and 10 random seeds. Each point corresponds to a different random seed, and the color of the point corresponds to the social group $S^{(b)}$. We compute $\widehat{\Delta}^{(b)}$ and $\Delta^{(b)}$ post-sigmoid using only the LP scores over the sampled (positive and negative) test edges.
    The plots display the NRMSE and PCC of $\widehat{\Delta}^{(b)}$ as a predictor of $\Delta^{(b)}$. 
    }
    \label{fig:delta-exp}
\end{figure*}

\begin{enumerate}[topsep=0pt,itemsep=0pt,parsep=0pt,partopsep=0pt,leftmargin=*]
\item Our theoretical analysis predicts that the LP score between two nodes $i, j$ that belong to the same social group $S^{(b)}$ will always be non-negative; however, $\Phi_s$ can predict negative scores for pairs of nodes in the same social group. In this case, it appears that $\Phi_s$ relies more on the dissimilarity of (transformed) features than node degree.
\item For many network datasets (especially from the citation and online social domains), there exist node pairs (near the origin) for which the theoretic LP score underestimates the $\Phi_s$ score. Upon further analysis (cf. Appendix \ref{sec:theory-pitfalls}), we find that the theoretic score is less predictive of the $\Phi_s$ score for nodes $i, j$
when the product of their degrees (i.e., their PA score) or similarity of their features is relatively low.
\item It appears that the theoretic LP score tends to poorly estimate the $\Phi_s$ score when the $\Phi_s$ score is relatively high; this suggests that $\Phi_s$ may conservatively rely more on the (dis)similarity of node features than node degree when the degree is large.
\end{enumerate}

We do not observe that the theoretic LP scores are strong predictors of the $\Phi_r$ scores, although there is still a moderate association between these variables. This could be because the error bound for the theoretic scores for $\Phi_r$, unlike for $\Phi_s$, has an extra dependence $\max_{u, v \in {\cal V}} \sqrt{\frac{\widehat{\mD}_{v v}}{\widehat{\mD}_{u u}}}$ on the degrees of the incident nodes (cf. $\zeta_r$ in Theorem \ref{thm:rw}). In contrast, the error bound for the theoretic scores for $\Phi_s$ (cf. $\zeta_s$ in Theorem \ref{thm:sym}) does not depend on this degree ratio. This ratio can be quite large in social networks (e.g., celebrities vs. new users in the Twitter follow network); we further confirm that this ratio is large for our datasets in \S\ref{sec:rw-error-analysis}.

\subsection{Within-Group Fairness}
\label{sec:fairness-implications-exp}

We now empirically validate the implications of GCN's PA bias for within-group unfairness in LP. We run experiments on three network datasets: (1) the NBA social network \citep{Dai2021SayNo}, (2) the German credit network \citep{Agarwal2021TowardsAU}, and (3) a new DBLP-Fairness citation network that we construct. We describe these datasets in \S\ref{sec:fairness-datasets}, including $\{ S^{(b)} \}_{b \in [B]}$ and $\{ T^{(d)} \}_{d \in [D]}$.

We train 2-layer GCN encoders $\Phi_s$ for LP (cf. \S\ref{sec:models}). In Figure \ref{fig:delta-exp}, for all the datasets, we plot $\widehat{\Delta}^{(b)}$ vs. $\Delta^{(b)}$ (cf. Eqns. \ref{eqn:fairness-metric}, \ref{eqn:approx-fairness-metric-sym}) for each $b \in [B]$. We qualitatively and quantitatively observe that $\widehat{\Delta}^{(b)}$ is moderately predictive of $\Delta^{(b)}$ for each dataset. This confirms our theoretical intuition (\S\ref{sec:fairness-implications-theory}) that a large disparity in the degree of nodes in $S^{(b)} \cap T^{(1)}$ vs. $S^{(b)} \cap T^{(2)}$ can greatly increase the unfairness $\Delta^{(b)}$ of $\Phi_s$ LP; such unfairness can amplify degree disparities, worsening power imbalances in the network. Many points deviate from the line of equality; these deviations can be explained by the reasons in \S\ref{sec:validating-theory}
and the compounding of errors.

\begin{table*}[!ht]
\caption{$\frac{1}{B} \sum_{b \in [B]} \Delta^{(b)}$ and the test AUC for the NBA, German, and DBLP-Fairness datasets with various settings of $\lambda_{\text{fair}}$. The \textbf{left} table corresponds to $\Phi_s$, and the \textbf{right} to $\Phi_r$.}
\label{tbl:fairness-results}
\centering
\begin{adjustbox}{max width=0.48\textwidth}
\begin{tabular}{lrrr}
\toprule
              & $\lambda_{\text{fair}}$ & $\frac{1}{B} \sum_{b \in [B]} \Delta^{(b)}$ ($\downarrow$) & $\Phi_s$ \textbf{Test AUC} ($\uparrow$) \\
\midrule
          NBA &                     4.0 &                       $0.000 \pm 0.001$ &              $0.753 \pm 0.002$ \\
          NBA &                     2.0 &                       $0.004 \pm 0.003$ &              $0.752 \pm 0.003$ \\
          NBA &                     1.0 &                       $0.007 \pm 0.004$ &              $0.752 \pm 0.003$ \\
          NBA &                     0.0 &                       $0.013 \pm 0.005$ &              $0.752 \pm 0.003$ \\
          \midrule
 DBLPFAIRNESS &                     4.0 &                       $0.072 \pm 0.018$ &              $0.741 \pm 0.008$ \\
 DBLPFAIRNESS &                     2.0 &                       $0.095 \pm 0.025$ &              $0.756 \pm 0.007$ \\
 DBLPFAIRNESS &                     1.0 &                       $0.110 \pm 0.033$ &              $0.770 \pm 0.010$ \\
 DBLPFAIRNESS &                     0.0 &                       $0.145 \pm 0.020$ &              $0.778 \pm 0.007$ \\
 \midrule
       GERMAN &                     4.0 &                       $0.012 \pm 0.006$ &              $0.876 \pm 0.017$ \\
       GERMAN &                     2.0 &                       $0.028 \pm 0.017$ &              $0.889 \pm 0.017$ \\
       GERMAN &                     1.0 &                       $0.038 \pm 0.016$ &              $0.897 \pm 0.014$ \\
       GERMAN &                     0.0 &                       $0.045 \pm 0.013$ &              $0.912 \pm 0.009$ \\
\bottomrule
\end{tabular}
\end{adjustbox}
\quad
\begin{adjustbox}{max width=0.48\textwidth}
\begin{tabular}{lrrr}
\toprule
              & $\lambda_{\text{fair}}$ & $\frac{1}{B} \sum_{b \in [B]} \Delta^{(b)}$ ($\downarrow$) & $\Phi_r$ \textbf{Test AUC} ($\uparrow$) \\
\midrule
 NBA & 4.0 & $0.000 \pm 0.000$ & $0.585 \pm 0.030$ \\
NBA & 2.0 & $0.000 \pm 0.000$ & $0.584 \pm 0.032$ \\
NBA & 1.0 & $0.000 \pm 0.000$ & $0.581 \pm 0.034$ \\
NBA & 0.0 & $0.000 \pm 0.000$ & $0.583 \pm 0.034$ \\
\midrule
DBLPFAIRNESS & 4.0 & $0.053 \pm 0.015$ & $0.715 \pm 0.010$ \\
DBLPFAIRNESS & 2.0 & $0.060 \pm 0.016$ & $0.731 \pm 0.009$ \\
DBLPFAIRNESS & 1.0 & $0.065 \pm 0.022$ & $0.746 \pm 0.009$ \\
DBLPFAIRNESS & 0.0 & $0.090 \pm 0.028$ & $0.758 \pm 0.011$ \\
\midrule
GERMAN & 4.0 & $0.029 \pm 0.011$ & $0.830 \pm 0.024$ \\
GERMAN & 2.0 & $0.031 \pm 0.019$ & $0.843 \pm 0.027$ \\
GERMAN & 1.0 & $0.019 \pm 0.012$ & $0.864 \pm 0.020$ \\
GERMAN & 0.0 & $0.015 \pm 0.005$ & $0.883 \pm 0.009$ \\
\bottomrule
\end{tabular}
\end{adjustbox}
\end{table*}

\subsection{Fairness Regularizer}
\label{sec:fairness-regularizer}

We evaluate our solution to alleviate LP unfairness (\S\ref{sec:fairness-implications-theory}). In particular, we add our fairness regularization term ${\cal L}_{\text{fair}}$ to the original training loss for the 2-layer $\Phi_s$ and $\Phi_r$ encoders. During each training epoch, we compute $\Delta^{(b)}$ post-sigmoid using only the LP scores over the sampled (positive and negative) training edges. In Table \ref{tbl:fairness-results}, we summarize the link prediction fairness $\left( \frac{1}{B} \sum_{b \in [B]} \Delta^{(b)} \right)$ and performance (test AUC) for the NBA, German, and DBLP-Fairness datasets with various settings of $\lambda_{\text{fair}}$.

For both graph filter types, we generally observe a significant decrease in $\frac{1}{B} \sum_{b \in [B]} \Delta^{(b)}$ (without a severe drop in test AUC) for $\lambda_{\text{fair}} > 0.0$ over $\lambda_{\text{fair}} = 0.0$ (with the exception of $\Phi_r$ for German); however, the varying magnitudes by which $\frac{1}{B} \sum_{b \in [B]} \Delta^{(b)}$ decreases across the datasets suggests that $\lambda_{\text{fair}}$ may need to be tuned per dataset. As expected, we mostly observe a tradeoff between $\frac{1}{B} \sum_{b \in [B]} \Delta^{(b)}$ and the test AUC as $\lambda_{\text{fair}}$ increases. Our experiments reveal that, regardless of graph filter type, even simple regularization approaches can alleviate this new form of unfairness. As this form of unfairness has not been previously explored, we have no baselines.

Our fairness regularizer can be easily integrated into model training, does not require significant additional computation, and directly optimizes for LP fairness. The time complexity of calculating the regularization term is ${\cal O} \left(\sum_{b = 1}^B  |S^{(b)} \cap T^{(1)}| \cdot | S^{(b)} | +  |S^{(b)} \cap T^{(2)}| \cdot | S^{(b)} | \right)$, as we have already computed the LP scores for the cross-entropy loss term and simply need to sum them appropriately with respect to the groups and subgroups. Furthermore, the time complexity of computing gradients for the regularization term is on the same order as backpropagation for the cross-entropy loss term.

However, our fairness regularizer is not applicable in settings where model parameters cannot be retrained or finetuned. Hence, we encourage future research to also explore post-processing fairness strategies. For example, for $\Phi_s$ models, based on our theory (cf. Theorem \ref{thm:sym}), for each pair of nodes $i, j$, we can decay the influence of GCN's PA bias by scaling (pre-activation) LP scores by $\left( \sqrt{\widehat{\mD}_{i i} \widehat{\mD}_{j j}} \right)^{-\alpha}$, where $0 < \alpha < 1$ is a hyperparameter that can be tuned to achieve a desirable balance between $\frac{1}{B} \sum_{b \in [B]} \Delta^{(b)}$ and the test AUC.

Empirical evaluation of our fairness regularizer using existing LP fairness metrics, such as statistical parity and equal opportunity dyadic fairness \citep{li2021on}, or equal opportunity degree bias \citep{Wang2022LinkBias}, is beyond the scope of our paper given that our algorithm and metric are designed to handle a different form of unfairness. For example, inter-group and intra-group links can be predicted at the same rate or with the same accuracy, but these links can be exclusively with high-degree nodes, thereby marginalizing low-degree nodes (cf. \S\ref{sec:pa-motivation}). Similarly, even if we consistently predict links with the same accuracy across nodes with different degrees, high-degree nodes can still receive higher LP scores than low-degree nodes (cf. \S\ref{sec:deg-bias-comments}).

\section{Conclusion}
We theoretically and empirically show that GCNs can have a PA bias in LP. We analyze how this bias can engender within-group unfairness, and amplify degree and power imbalances in networks. We further propose a simple training-time strategy to alleviate this unfairness. We encourage future work to: (1) explore PA bias in other GNN architectures and directed and heterophilic networks, (2) characterize the ``rich get richer'' evolution of networks affected by GCN's PA bias, and (3) propose pre-processing and post-processing strategies for within-group LP unfairness.

Because this unfairness is at the level of dyads, we would like to explore new forms of unfairness that occur at the level of higher-order structures (e.g., prediction disparities between important coalitions of nodes). Moreover, node degree is a local property, and it would be valuable to theoretically and empirically relate higher-order graph properties (e.g., local clustering coefficient, different measures of centrality) to unfairness.

\section*{Acknowledgements}
We would like to thank the anonymous reviewers for their feedback on this work. This work was partially supported by NSF 2211557, NSF 1937599, NSF 2119643, NSF 2303037, NSF 2312501, NASA, SRC JUMP 2.0 Center, Amazon Research Awards, and Snapchat Gifts.

\section*{Impact Statement}

Our paper seeks to uncover and combat discrimination, bias, and unfairness in GNNs. Throughout, we tie our analysis back to issues of disparity and power, towards advancing justice in graph learning. While we propose a strategy to alleviate LP unfairness, we emphasize that it is not a `silver bullet' solution; we encourage graph learning practitioners to adopt a sociotechnical approach to fairness and continually adapt their algorithms, datasets, and metrics in response to the everchanging landscape of inequality and power. Furthermore, the fairness of GCN LP should not sidestep concerns about GCN LP being used \textit{at all} in certain scenarios.

Some datasets that we use contain protected attribute information (detailed in \S\ref{sec:fairness-datasets}). We avoid using datasets that enable carceral technology (e.g., Recidivism \citep{Agarwal2021TowardsAU}). We release our code and data with an MIT license.

For transparency, we do our best to discuss limitations throughout the paper. For each lemma and theorem (\S\ref{sec:theoretical-analysis}), our assumptions are clearly explained and justified either before or in the statement thereof, and we include complete proofs of our theoretical claims in \S\ref{sec:proofs} and \S\ref{sec:approx-delta-proof}.

For reproducibility, we provide all our code and data (including the raw DBLP-Fairness dataset) in our GitHub repository, along with a README.
We detail our data processing steps in \S\ref{sec:dblp-fairness}. Furthermore, our experiments (\S\ref{sec:experiments}) are run with 10 random seeds and errors are reported. We provide model implementation details in \S\ref{sec:models}.

\bibliography{references}
\bibliographystyle{icml2024}

\newpage
\appendix
\onecolumn

\clearpage

\appendix

\begin{center}
    \Large \textbf{Supplementary Text}
\end{center}

\section{Proofs}
\label{sec:proofs}

\subsection{Proof of Lemma \ref{lemma:taylor}}

\begin{proof}
Similarly to \citet{Xu2018RepresentationLO, tang2020degree}, we compute the first-order partial derivatives of $\Phi_s$ and $\Phi_r$:
\begin{align}
\frac{\partial \vs_i^{(L)}}{\partial \vx_j} = \sum_{p \in \Psi^{L + 1}_{i \to j}} \prod_{l = L}^{1} \frac{\text{diag} \left(\mathbbm{1}_{\vz^{(l)}_{p^{(l)}} > 0}\right) \mW_s^{(l)}}{\sqrt{\mD_{p^{(l)} p^{(l)}} \mD_{p^{(l - 1)} p^{(l - 1)}}}} &, \quad
\frac{\partial \vr_i^{(L)}}{\partial \vx_j} = \sum_{p \in \Psi^{L + 1}_{i \to j}} \prod_{l = L}^{1} \frac{\text{diag} \left(\mathbbm{1}_{\vz^{(l)}_{p^{(l)}} > 0}\right) \mW_s^{(l)}}{\mD_{p^{(l)} p^{(l)}}} \\
\frac{\partial \vs_i^{(L)}}{\partial \vx_j} = \sqrt{\frac{\mD_{ii}}{\mD_{jj}}} \sum_{p \in \Psi^{L + 1}_{i \to j}} &\prod_{l = L}^{1} \frac{\text{diag} \left(\mathbbm{1}_{\vz^{(l)}_{p^{(l)}} > 0}\right) \mW_s^{(l)}}{\mD_{p^{(l)} p^{(l)}}}
\end{align}
where $p^{(l)}$ is the $l$-th node on path $p$ in the computation graph of $\Phi_s$ or $\Phi_r$ ($p^{(L)}$ is node $i$ and $p^{(0)}$ is node $j$); $\Psi^\gamma_{i \to j}$ is the set of all $\gamma$-length random walk paths from node $i$ to $j$; and $\vz^{(l)}_{p^{(l)}}$ is pre-activated $\vs^{(l)}_{p^{(l)}}$ or $\vr^{(l)}_{p^{(l)}}$.

With our assumption that the path from node $i \to j$ in the computation graph of $\Phi_s$ is independently activated with probability $\rho_s (i)$, and similarly, $\rho_r (i)$ for $\Phi_r$:
\begin{small}
\begin{align}
&\mathop{\mathbb{E}} \left[ \frac{\partial \vs_i^{(L)}}{\partial \vx_j} \right] = \left( \mD^{-\frac{1}{2}} \mA \mD^{-\frac{1}{2}} \right)^L_{i j} \rho_s (i) \left(\prod_{l = L}^1 \mW_s^{(l)} \right), \\
&\mathop{\mathbb{E}} \left[ \frac{\partial \vr_i^{(L)}}{\partial \vx_j} \right] = \left( \mD^{-1} \mA \right)^L_{i j} \rho_r (i) \left(\prod_{l = L}^1 \mW_r^{(l)} \right).
\end{align}
\end{small}

Then, recalling Eqn. \ref{eqn:taylor}:
\begin{small}
\begin{align}
&\mathop{\mathbb{E}} \left[ \vs_i^{(L)} \right] = \sum_{j \in {\cal V}} \left( \mD^{-\frac{1}{2}} \mA \mD^{-\frac{1}{2}} \right)^L_{i j} \rho_s (i) \left(\prod_{l = L}^1 \mW_s^{(l)} \right) \vx_j + \mathbf{0}, \\
&\mathop{\mathbb{E}} \left[ \vr_i^{(L)} \right] = \sum_{j \in {\cal V}} \left( \mD^{-1} \mA \right)^L_{i j} \rho_r (i) \left(\prod_{l = L}^1 \mW_r^{(l)} \right) \vx_j + \mathbf{0} \\
&\mathop{\mathbb{E}} \left[ \vs_i^{(L)} \right] = \sum_{j \in {\cal V}} \rho_s (i) \left( \mD^{-\frac{1}{2}} \mA \mD^{-\frac{1}{2}} \right)^L_{i j} \alpha_j, \quad 
\mathop{\mathbb{E}} \left[ \vr_i^{(L)} \right] = \sum_{j \in {\cal V}} \rho_r (i) \left( \mD^{-1} \mA \right)^L_{i j} \beta_j.
\end{align}
\end{small}
\end{proof}

\clearpage

\subsection{Proof of Lemma \ref{lemma:oversmoothing-sym}}

\begin{proof}
For $j \in S^{(b)}$, we can re-express
$\widehat{\mP}^L_{i j} = \left( \widehat{\mP}^{(b)} \right)^L_{i j} = \left( \ve^{(i)} \right)^\intercal \left( \widehat{\mP}^{(b)} \right)^L \ve^{(j)}$\footnote{For simplicity, we abuse notation here: $ \left( \widehat{\mP}^{(b)} \right)^L_{i j}$ is not the entry at row $i$ and column $j$, but rather the entry at the row corresponding to node $i$ and column corresponding to node $j$. Similarly, $\ve^{(i)}$ is the standard basis vector with a 1 at the entry corresponding to node $i$.}. By the spectral properties of $\widehat{\mP}^{(b)}$, $\left( \ve^{(i)} \right)^\intercal \vv_1^{(b)} = \sqrt{\frac{ \widehat{\mD}_{i i}}{\text{vol} \left({\cal G}^{(b)}\right)}}$ \citep{Lovsz2001RandomWO}. Hence:

\begin{align}
\widehat{\mP}^L_{i j} &= \sum_{k = 1}^{\left| S^{(b)} \right|} \left( \lambda_k^{(b)} \right)^L \left( \ve^{(i)} \right)^\intercal \vv^{(b)}_k \left( \vv^{(b)}_k \right)^\intercal \ve^{(j)} \\
&= \frac{ \sqrt{\widehat{\mD}_{i i} \widehat{\mD}_{j j}}}{\text{vol} \left({\cal G}^{(b)}\right)} + \sum_{k = 2}^{\left| S^{(b)} \right|} \left( \lambda_k^{(b)} \right)^L \left( \ve^{(i)} \right)^\intercal \vv^{(b)}_k \left( \vv^{(b)}_k \right)^\intercal \ve^{(j)}
\end{align}

Then, by Cauchy-Schwarz:

\begin{small}
\begin{align}
\left| \widehat{\mP}^L_{i j} - \frac{ \sqrt{\widehat{\mD}_{i i} \widehat{\mD}_{j j}}}{\text{vol} \left({\cal G}^{(b)}\right)} \right|
&\leq \left( \lambda^{(b)} \right)^L \sum_{k = 1}^{\left| S^{(b)} \right|} \left| \left( \ve^{(i)} \right)^\intercal \vv^{(b)}_k \right| \left| \left( \ve^{(j)} \right)^\intercal \vv^{(b)}_k  \right| \\
&\leq \left( \lambda^{(b)} \right)^L \left( \sum_{k = 1}^{\left| S^{(b)} \right|} \left| \left( \ve^{(i)} \right)^\intercal \vv^{(b)}_k \right|^2 \right)^{\frac{1}{2}} \left( \sum_{k = 1}^{\left| S^{(b)} \right|} \left| \left( \ve^{(j)} \right)^\intercal \vv^{(b)}_k \right|^2 \right)^{\frac{1}{2}} \\
&= \left( \lambda^{(b)} \right)^L \left( \left( \ve^{(i)} \right)^\intercal \mV^{(b)} \left( \mV^{(b)} \right)^\intercal \ve^{(i)} \right)^{\frac{1}{2}} \left( \left(\ve^{(j)}\right)^\intercal \mV^{(b)} \left( \mV^{(b)} \right)^\intercal \ve^{(j)} \right)^{\frac{1}{2}} \\
&= \left( \lambda^{(b)} \right)^L \left\| \ve^{(i)} \right\|_2 \left\| \ve^{(j)} \right\|_2 \\
&= \left( \lambda^{(b)} \right)^L
\end{align}
\end{small}

Let $\mP^L = \left(\widehat{\mP} + \Xi^{(0)} \right)^L = \widehat{\mP}^L + \Xi^{(L)}$. Then, by the triangle inequality:
\begin{align}
\left| \mP^L_{i j} - \frac{ \sqrt{\widehat{\mD}_{i i} \widehat{\mD}_{j j}}}{\text{vol} \left({\cal G}^{(b)}\right)} \right| &\leq \left( \lambda^{(b)} \right)^L + \left| \left( \ve^{(i)} \right)^\intercal \Xi^{(L)} \ve^{(j)} \right| \\
 &\leq \left( \lambda^{(b)} \right)^L + \left\| \Xi^{(L)} \right\|_{op} \\
 &\leq \left( \lambda^{(b)} \right)^L + \sum_{l = 1}^L {L \choose l} \left\| \Xi^{(0)} \right\|^l_{op} \left\| \widehat{\mP} \right\|^{L - l}_{op}
\end{align}

For $j \notin S^{(b)}$, $\widehat{\mP}^L_{i j} = 0$. Then: 
\begin{align}
\left| \mP^L_{i j} - 0 \right| &\leq \left| \left( \ve^{(i)} \right)^\intercal \Xi^{(L)} \ve^{(j)} \right| \\
 &\leq \sum_{l = 1}^L {L \choose l} \left\| \Xi^{(0)} \right\|^l_{op} \left\| \widehat{\mP} \right\|^{L - l}_{op}
\end{align}
\end{proof}

\clearpage

\subsection{Proof of Theorem \ref{thm:sym}}

\begin{proof}
For $u, v \in {\cal V}$, let $\left| \delta_{u v} \right| \leq \zeta_s$. Combining Lemmas \ref{lemma:taylor} and \ref{lemma:oversmoothing-sym}, by our assumption that the computation graph paths to $i, j$ are activated independently:
\begin{small}
\begin{align}
&\mathop{\mathbb{E}} \left[ f_{LP} \left( \vs_i^{(L)}, \vs_j^{(L)}\right) \right] = \mathop{\mathbb{E}} \left[ \vs_i^{(L)} \right]^\intercal \mathop{\mathbb{E}} \left[ \vs_j^{(L)} \right] \\
&= \overline{\rho}_s^2 (b) \left( \sum_{k \in S^{(b)}} \frac{\sqrt{\widehat{\mD}_{i i} \widehat{\mD}_{k k}}}{\text{vol} \left({\cal G}^{(b)}\right)} \alpha_k + \sum_{k \in {\cal V}} \delta_{i k} \alpha_k \right)^\intercal \left( \sum_{k \in S^{(b)}} \frac{\sqrt{\widehat{\mD}_{j j} \widehat{\mD}_{k k}}}{\text{vol} \left({\cal G}^{(b)}\right)} \alpha_k + \sum_{k \in {\cal V}} \delta_{j k} \alpha_k \right) \\
&= \overline{\rho}_s^2 (b) \sqrt{\widehat{\mD}_{i i} \widehat{\mD}_{j j}} \underbrace{\left\| \sum_{k \in S^{(b)}} \frac{ \sqrt{\widehat{\mD}_{k k}}}{\text{vol} ({\cal G}^{(b)})} \alpha_k \right\|^2_2}_{\geq 0} \\
&+ \overline{\rho}_s^2 (b) \left( \sqrt{\widehat{\mD}_{i i}} \sum_{k \in S^{(b)}} \frac{ \sqrt{\widehat{\mD}_{k k}}}{\text{vol} ({\cal G}^{(b)})} \alpha_k \right)^\intercal \left( \sum_{k \in {\cal V}} \delta_{j k} \alpha_k \right) \\
&+ \overline{\rho}_s^2 (b) \left( \sum_{k \in {\cal V}} \delta_{i k} \alpha_k \right)^\intercal \left( \sqrt{\widehat{\mD}_{j j}} \sum_{k \in S^{(b)}} \frac{ \sqrt{\widehat{\mD}_{k k}}}{\text{vol} ({\cal G}^{(b)})} \alpha_k \right)  \\
&+ \overline{\rho}_s^2 (b) \left( \sum_{k \in {\cal V}} \delta_{i k} \alpha_k \right)^\intercal \left( \sum_{k \in {\cal V}} \delta_{j k} \alpha_k \right) 
\end{align}
\end{small}

Then, by Cauchy-Schwarz and the triangle inequality:
\begin{small}
\begin{align}
&\left| \mathop{\mathbb{E}} \left[ f_{LP} \left( \vs_i^{(L)}, \vs_j^{(L)}\right) \right] - \underbrace{\overline{\rho}_s^2 (b) \left\| \sum_{k \in S^{(b)}} \frac{ \sqrt{\widehat{\mD}_{k k}}}{\text{vol} ({\cal G}^{(b)})} \alpha_k \right\|^2_2 \sqrt{\widehat{\mD}_{i i} \widehat{\mD}_{j j}}}_{\propto \sqrt{\widehat{\mD}_{i i} \widehat{\mD}_{j j}}} \right| \\
&\leq \zeta_s \overline{\rho}_s^2 (b) \left( \sqrt{\widehat{\mD}_{i i}} + \sqrt{\widehat{\mD}_{j j}} \right) \left\| \sum_{k \in S^{(b)}} \frac{ \sqrt{\widehat{\mD}_{k k}}}{\text{vol} ({\cal G}^{(b)})} \alpha_k \right\|_2 \left( \sum_{k \in {\cal V}} \| \alpha_k \|_2 \right) + \zeta_s^2 \overline{\rho}_s^2 (b) \left( \sum_{k \in {\cal V}} \| \alpha_k \|_2 \right)^2
\end{align}
\end{small}
\end{proof}

\clearpage

\subsection{Lemma \ref{lemma:oversmoothing-rw} and Proof}

\begin{lemma}
\label{lemma:oversmoothing-rw}
We introduce the notation $\mP = \mD^{-1} \mA$. We further define $\widehat{\mP} = \widehat{\mD}^{-1} \widehat{\mA}$. Fix $i \in S^{(b)}$. Then, for $j \in S^{(b)}$:
\begin{align}
\left| \mP^L_{i j} - \frac{ \widehat{\mD}_{j j}}{\text{vol} \left({\cal G}^{(b)}\right)} \right| \leq \sqrt{\frac{\widehat{\mD}_{j j}}{\widehat{\mD}_{i i}}} \left( \lambda^{(b)} \right)^L + \sum_{l = 1}^L {L \choose l} \left\| \Xi^{(0)} \right\|^l_{op} \left\| \widehat{\mP} \right\|^{L - l}_{op}
\end{align}
And for $j \notin S^{(b)}$:
\begin{align}
\left| \mP^L_{i j} - 0 \right| &\leq \sum_{l = 1}^L {L \choose l} \left\| \Xi^{(0)} \right\|^l_{op} \left\| \widehat{\mP} \right\|^{L - l}_{op}
\end{align}
\end{lemma}

\begin{proof}
Similar to the proof of Lemma \ref{lemma:oversmoothing-sym}:

\begin{align}
\widehat{\mP}^L_{i j} &= \frac{\widehat{\mD}_{j j}}{\text{vol} \left({\cal G}^{(b)}\right)} + \sqrt{\frac{\widehat{\mD}_{j j}}{\widehat{\mD}_{i i}}} \sum_{k = 2}^{\left| S^{(b)} \right|} \left( \lambda_k^{(b)} \right)^L \left( \ve^{(i)} \right)^\intercal \vv^{(b)}_k \left( \vv^{(b)}_k \right)^\intercal \ve^{(j)}
\end{align}

Subsequently:

\begin{small}
\begin{align}
\left| \widehat{\mP}^L_{i j} - \frac{\widehat{\mD}_{j j}}{\text{vol} \left({\cal G}^{(b)}\right)} \right|
&\leq \sqrt{\frac{\widehat{\mD}_{j j}}{\widehat{\mD}_{i i}}} \left( \lambda^{(b)} \right)^L
\end{align}
\end{small}

Finally:
\begin{align}
\left| \mP^L_{i j} - \frac{\widehat{\mD}_{j j}}{\text{vol} \left({\cal G}^{(b)}\right)} \right| 
 &\leq \zeta_r = \max_{u, v \in {\cal V}} \sqrt{\frac{\widehat{\mD}_{v v}}{\widehat{\mD}_{u u}}} \left( \lambda^{(b)} \right)^L + \sum_{l = 1}^L {L \choose l} \left\| \Xi^{(0)} \right\|^l_{op} \left\| \widehat{\mP} \right\|^{L - l}_{op}
\end{align}

For $j \notin S^{(b)}$, $\widehat{\mP}^L_{i j} = 0$. Then: 
\begin{align}
\left| \mP^L_{i j} - 0 \right| &\leq \sum_{l = 1}^L {L \choose l} \left\| \Xi^{(0)} \right\|^l_{op} \left\| \widehat{\mP} \right\|^{L - l}_{op} \leq \zeta_r
\end{align}
\end{proof}

\clearpage

\subsection{Proof of Theorem \ref{thm:rw}}

\begin{proof}
For $u, v \in {\cal V}$, let $\left| \delta_{u v} \right| \leq \zeta_r$. Combining Lemmas \ref{lemma:taylor} and \ref{lemma:oversmoothing-rw}, by our assumption that the computation graph paths to $i, j$ are activated independently:
\begin{small}
\begin{align}
&\mathop{\mathbb{E}} \left[ f_{LP} \left( \vr_i^{(L)}, \vr_j^{(L)}\right) \right]
= \mathop{\mathbb{E}} \left[ \vr_i^{(L)} \right]^\intercal \mathop{\mathbb{E}} \left[ \vr_j^{(L)} \right] \\
&= \overline{\rho}_r^2 (b) \left( \sum_{k \in S^{(b)}} \frac{\widehat{\mD}_{k k}}{\text{vol} \left({\cal G}^{(b)}\right)} \beta_k + \sum_{k \in {\cal V}} \delta_{i k} \beta_k \right)^\intercal \left( \sum_{k \in S^{(b)}} \frac{\widehat{\mD}_{k k}}{\text{vol} \left({\cal G}^{(b)}\right)} \beta_k + \sum_{k \in {\cal V}} \delta_{j k} \beta_k \right) \\
&= \overline{\rho}_r^2 (b) \underbrace{\left\| \sum_{k \in S^{(b)}} \frac{ \widehat{\mD}_{k k}}{\text{vol} ({\cal G}^{(b)})} \beta_k \right\|^2_2}_{\geq 0}
+ \overline{\rho}_r^2 (b) \left( \sum_{k \in S^{(b)}} \frac{ \widehat{\mD}_{k k}}{\text{vol} ({\cal G}^{(b)})} \beta_k \right)^\intercal \left( \sum_{k \in {\cal V}} \delta_{j k} \beta_k \right) \\
&+ \overline{\rho}_r^2 (b) \left( \sum_{k \in {\cal V}} \delta_{i k} \beta_k \right)^\intercal \left( \sum_{k \in S^{(b)}} \frac{ \widehat{\mD}_{k k}}{\text{vol} ({\cal G}^{(b)})} \beta_k \right)
+ \overline{\rho}_r^2 (b) \left( \sum_{k \in {\cal V}} \delta_{i k} \beta_k \right)^\intercal \left( \sum_{k \in {\cal V}} \delta_{j k} \beta_k \right) 
\end{align}
\end{small}

Then, by Cauchy-Schwarz and the triangle inequality:
\begin{small}
\begin{align}
&\left| \mathop{\mathbb{E}} \left[ f_{LP} \left( \vr_i^{(L)}, \vr_j^{(L)}\right) \right] - \underbrace{\overline{\rho}_r^2 (b) \left\| \sum_{k \in S^{(b)}} \frac{ \widehat{\mD}_{k k}}{\text{vol} ({\cal G}^{(b)})} \beta_k \right\|^2_2}_{\propto \text{ constant}} \right| \\
&\leq \zeta_r \overline{\rho}_r^2 (b) \left\| \sum_{k \in S^{(b)}} \frac{ \widehat{\mD}_{k k}}{\text{vol} ({\cal G}^{(b)})} \beta_k \right\|_2 \left( \sum_{k \in {\cal V}} \| \beta_k \|_2 \right) + \zeta_r^2 \overline{\rho}_r^2 (b) \left( \sum_{k \in {\cal V}} \| \beta_k \|_2 \right)^2
\end{align}
\end{small}
\end{proof}

\clearpage

\section{Approximation of $\Delta^{(b)}$}
\label{sec:approx-delta-proof}

\subsection{Approximation of $\Delta^{(b)}$ for $\Phi_s$}
\label{sec:approx-delta-sym-proof}

\begin{small}
\begin{align}
&\Delta^{(b)} \left( \vs_i^{(L)}, \vs_j^{(L)}\right) \\
&= \big| \frac{1}{\left| (S^{(b)} \cap T^{(1)}) \times S^{(b)} \right|} \sum_{i \in S^{(b)} \cap T^{(1)}} \sum_{j \in S^{(b)}} f_{LP} \left( \vs_i^{(L)}, \vs_j^{(L)}\right) \\
&- \frac{1}{\left| (S^{(b)} \cap T^{(2)}) \times S^{(b)} \right|} \sum_{i \in S^{(b)} \cap T^{(2)}} \sum_{j \in S^{(b)}} f_{LP} \left( \vs_i^{(L)}, \vs_j^{(L)}\right) \big| \\
&\approxeq \big| \frac{1}{\left| S^{(b)} \cap T^{(1)} \right| \left| S^{(b)} \right|} \sum_{i \in S^{(b)} \cap T^{(1)}} \sum_{j \in S^{(b)}} \overline{\rho}_s^2 (b) \sqrt{\widehat{\mD}_{i i} \widehat{\mD}_{j j}} \left\| \sum_{k \in S^{(b)}} \frac{ \sqrt{\widehat{\mD}_{k k}}}{\text{vol} ({\cal G}^{(b)})} \alpha_k \right\|^2_2  \\
&- \frac{1}{\left| S^{(b)} \cap T^{(2)} \right| \left| S^{(b)} \right|} \sum_{i \in S^{(b)} \cap T^{(2)}} \sum_{j \in S^{(b)}} \overline{\rho}_s^2 (b) \sqrt{\widehat{\mD}_{i i} \widehat{\mD}_{j j}} \left\| \sum_{k \in S^{(b)}} \frac{ \sqrt{\widehat{\mD}_{k k}}}{\text{vol} ({\cal G}^{(b)})} \alpha_k \right\|^2_2 \big| \\
&= \frac{\overline{\rho}_s^2 (b)}{\left| S^{(b)} \right|} \left\| \sum_{k \in S^{(b)}} \frac{ \sqrt{\widehat{\mD}_{k k}}}{\text{vol} ({\cal G}^{(b)})} \alpha_k \right\|^2_2 \left| \sum_{j \in S^{(b)}} \sqrt{\widehat{\mD}_{j j}}  \underbrace{\left( \mathop{\mathbb{E}}_{i \sim U(S^{(b)} \cap T^{(1)})}  \sqrt{\widehat{\mD}_{i i}} - \mathop{\mathbb{E}}_{i \sim U(S^{(b)} \cap T^{(2)})} \sqrt{\widehat{\mD}_{i i}} \right) }_{\text{degree disparity}}  \right|
\end{align}
\end{small}

\subsection{Approximation of $\Delta^{(b)}$ for $\Phi_r$}
\label{sec:approx-delta-rw-proof}

\begin{small}
\begin{align}
&\Delta^{(b)} \left( \vr_i^{(L)}, \vr_j^{(L)}\right) \\
&= \big| \frac{1}{\left| (S^{(b)} \cap T^{(1)}) \times S^{(b)} \right|} \sum_{i \in S^{(b)} \cap T^{(1)}} \sum_{j \in S^{(b)}} f_{LP} \left( \vr_i^{(L)}, \vr_j^{(L)}\right) \\
&- \frac{1}{\left| (S^{(b)} \cap T^{(2)}) \times S^{(b)} \right|} \sum_{i \in S^{(b)} \cap T^{(2)}} \sum_{j \in S^{(b)}} f_{LP} \left( \vr_i^{(L)}, \vr_j^{(L)}\right) \big| \\
&\approxeq \big| \frac{1}{\left| S^{(b)} \cap T^{(1)} \right| \left| S^{(b)} \right|} \sum_{i \in S^{(b)} \cap T^{(1)}} \sum_{j \in S^{(b)}} \overline{\rho}_r^2 (b) \left\| \sum_{k \in S^{(b)}} \frac{ \widehat{\mD}_{k k}}{\text{vol} ({\cal G}^{(b)})} \beta_k \right\|^2_2  \\
&- \frac{1}{\left| S^{(b)} \cap T^{(2)} \right| \left| S^{(b)} \right|} \sum_{i \in S^{(b)} \cap T^{(2)}} \sum_{j \in S^{(b)}} \overline{\rho}_r^2 (b) \left\| \sum_{k \in S^{(b)}} \frac{ \widehat{\mD}_{k k}}{\text{vol} ({\cal G}^{(b)})} \beta_k \right\|^2_2 \big| \\
&= 0
\end{align}
\end{small}

\clearpage

\section{Datasets Used in \S\ref{sec:validating-theory}}
\label{sec:validating-datasets}

In our experiments in \S\ref{sec:validating-theory}, we use 10 real-world network datasets from \citet{bojchevski2018deep}, \citet{Shchur2018PitfallsOG}, \citet{rozemberczki2020characteristic}, and \citet{rozemberczki2021multi}, covering diverse domains (e.g., citation networks, collaboration networks, online social networks). We provide a description and some statistics of each dataset in Table \ref{tab:datasets}. All the datasets have node features and are undirected. We were unable to find the exact class names and their label correspondence from the dataset documentation.

\begin{itemize}
    \item In all the citation network datasets, nodes represent documents, edges represent citation links, and features are a bag-of-words representation of documents. We row-normalize the features to sum to 1, following \citet{Fey/Lenssen/2019}\footnote{\url{https://github.com/pyg-team/pytorch_geometric/blob/master/examples/link_pred.py}}. The classification task is to predict the topic of documents.
    \item In the collaboration network datasets, nodes represent authors, edges represent coauthorships, and features are embeddings of paper keywords for authors' papers. The classification task is to predict the most active field of study for authors.
    \item In the LastFMAsia network dataset, nodes represent LastFM users from Asia, edges represent friendships between users, and features are embeddings of the artists liked by users. The classification task is to predict the home country of users.
    \item In the Twitch network datasets, nodes represent gamers on Twitch, edges represent followerships between them, and features are embeddings of the history of games played by the Twitch users. The classification task is to predict whether or not a gamer streams adult content.
\end{itemize}

We only run experiments on datasets that can fit without sampling nodes on a single NVIDIA GeForce GTX Titan Xp Graphic Card with 12196MiB of space.
Furthermore, we only consider the three largest datasets (i.e., with the most nodes) from \citet{rozemberczki2021multi}. We use PyTorch Geometric to load and process all datasets \citep{Fey/Lenssen/2019}.

\begin{table}[!ht]
    \caption{Summary of the datasets used in our experiments.}
    \label{tab:datasets}
    \centering
    \begin{tabular}{l|c|c|c|c|c}
    \toprule
    \textbf{Name} & \textbf{Domain} & \textbf{\# Nodes} & \textbf{\# Edges} & \textbf{\# Features} & \textbf{\# Classes}  \\
    \midrule
    Cora & citation & 19793 & 126842 & 8710 & 70 \\
    CiteSeer & citation & 4230 & 10674 & 602 & 6 \\
    DBLP & citation & 17716 & 105734 & 1639 & 4 \\
    PubMed & citation & 19717 & 88648 & 500 & 3 \\
    \midrule
    CS & collaboration & 18333 & 163788 & 6805 & 15 \\
    Physics & collaboration & 34493 & 495924 & 8415 & 5 \\
    \midrule
    LastFMAsia & online social & 7624 & 55612 & 128 & 18 \\
    Twitch-DE & online social & 9498 & 315774 & 128 & 2 \\
    Twitch-EN & online social & 7126 & 77774 & 128 & 2 \\
    Twitch-FR & online social & 6551 & 231883 & 128 & 2 \\
    \bottomrule
    \end{tabular}
\end{table}

\clearpage

\section{Datasets Used in \S\ref{sec:fairness-implications-exp}}
\label{sec:fairness-datasets}

We run experiments on three network datasets: (1) the NBA social network (cf. \S\ref{sec:nba}), (2) the German credit network (cf. \S\ref{sec:german}), and (3) a new DBLP-Fairness citation network that we construct (cf. \S\ref{sec:dblp-fairness}). All the datasets have node features and are undirected. We do not pass sensitive attributes as features to the models that we train. For each dataset, we min-max normalize node features to fall in $[-1, 1]$, following \citet{Dai2021SayNo} and \citet{Agarwal2021TowardsAU}. Furthermore, for all datasets, $D = 2$.

\subsection{NBA Dataset}
\label{sec:nba}
The NBA network \citep{Dai2021SayNo} has 403 nodes representing NBA basketball players who are connected if they follow each other on Twitter. There are 21242 links. Each node has 95 features, with an average degree of $52.71 \pm 35.14$. We consider two sensitive attributes per node:
\begin{itemize}
    \item Age $\{ S^{(b)} \}_{b \in [B]}$: how old the payer is, i.e., \textsc{Young} ($\leq 25$ years) or \textsc{Old} ($> 25$ years).
    \item Nationality $\{ T^{(d)} \}_{d \in [D]}$: from where the player is, i.e., \textsc{United States} or \textsc{Overseas}.
    
\end{itemize} 

\subsection{German Dataset}
\label{sec:german}
The German network \citep{Agarwal2021TowardsAU} comprises 1000 nodes representing clients in a German bank who are connected if they have similar credit accounts. The German network is not natively a graph dataset; synthetic edges were created by \citeauthor{Agarwal2021TowardsAU} There are 44484 links. Each node has 27 features (e.g., loan amount, account-related features), with an average degree of $44.48 \pm 26.52$. We consider two sensitive attributes per node:
\begin{itemize}
    \item Foreign worker $\{ S^{(b)} \}_{b \in [B]}$: whether the client is a foreign worker, i.e., \textsc{Yes} or \textsc{No}.
    \item Gender $\{ T^{(d)} \}_{d \in [D]}$: the gender of the client, i.e., \textsc{Man} or \textsc{Woman}.
\end{itemize}

\subsection{DBLP-Fairness Dataset}
\label{sec:dblp-fairness}
In this subsection, we detail how we construct the DBLP-Fairness dataset. We build DBLP-Fairness, as there are only a few natively-graph network datasets with sensitive attributes that are appropriate for graph learning \citep{Subramonian2022}.

We begin with the version of the DBLP-Citation-network V12 dataset from \citep{Tang2008Arnet} that was processed by \citet{Xu2021TeamPD}. This dataset has 3658127 nodes. Each node represents a paper and each edge represents a citation link. We consider five node features:
\begin{itemize}
    \item Team size: the number of authors on the paper.
    \item Mean collaborators: the average number of collaborators with whom the authors have previously published.
    \item Gini collaborators: the Gini coefficient of the number of collaborators with whom the authors have previously published.
    \item Mean productivity: the average number of papers that the authors have previously published.
    \item Gini productivity: the Gini coefficient of the number of papers that the authors have previously published.
\end{itemize}

We also consider two sensitive attributes per node:
\begin{itemize}
    \item Field $\{ S^{(b)} \}_{b \in [B]}$: the field to which the paper belongs, i.e., \textsc{Programming Languages} or \textsc{Databases}.
    \item Nationality $\{ T^{(d)} \}_{d \in [D]}$: the country where most authors reside, i.e., \textsc{United States} or \textsc{China}.
\end{itemize}

In DBLP-Fairness, we only include papers whose nationality is \textsc{United States} or \textsc{China}; American and Chinese citation networks are known to be stratified \citep{zhao2022venue}. We also only include papers whose field is \textsc{Programming Languages} or \textsc{Databases}; we infer the field of a paper using its keywords (i.e., whether they contain ``programming language'' and ``database''), and discard papers which include both ``programming language'' and ``database'' in its keywords. Furthermore, we filter out all papers from before 2010.
We sought DBLB-Fairness to be of comparable size to the citation networks in \S\ref{sec:validating-datasets}. Following filtering, we were left with 14537 nodes and 24844 edges.  

\clearpage

\section{Models}
\label{sec:models}

For all experiments, we use GCN encoders \citep{kipf2017semisupervised} to get node representations. Each encoder has two layers (128-dimensional hidden layer, 64-dimensional output layer) with a ReLU nonlinearity in between. We only use two layers, as this is common practice in graph deep learning to prevent oversmoothing \citep{Oono2019OnAB}; however, we run experiments with four layers in \S\ref{sec:additional-experiments}. We do not use any regularization (e.g., Dropout, BatchNorm). The encoders are explicitly trained for LP with the inner-product LP score function in Eqn. \ref{eqn:score-function}, binary cross-entropy loss, and the Adam optimizer with full-batch gradient descent and a learning rate of 0.01 \citep{Kingma2014AdamAM}. We use a random link split of 0.85-0.05-0.1 for train-val-test, following the PyTorch Geometric LP example\footnote{\url{https://github.com/pyg-team/pytorch_geometric/blob/master/examples/link_pred.py}}. We train the encoders for 100 epochs, with a new round of negative link sampling during every epoch; we use a 1:1 ratio of positive to negative links. We ultimately select the model parameters with the highest validation ROC-AUC. Although we do not do any hyperparameter tuning, the test ROC-AUC values (displayed in the figures in \S\ref{sec:experiments}) indicate that the encoders are well-trained. We use PyTorch \citep{paszke2019torch} and PyTorch Geometric \citep{Fey/Lenssen/2019} to train all the encoders on a single NVIDIA GeForce GTX Titan Xp Graphic Card with 12196MiB of space.

\clearpage

\section{Remaining Plots}
\label{sec:remaining-plots}

\begin{figure*}[ht!]
    \centering
    \begin{subfigure}[t]{0.4\textwidth}
        \centering
        \includegraphics[width=\textwidth]{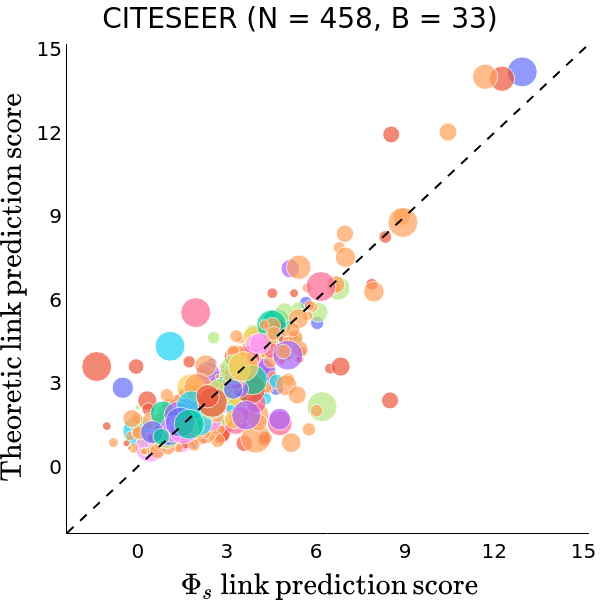}
    \end{subfigure}%
    \begin{subfigure}[t]{0.4\textwidth}
        \centering
        \includegraphics[width=\textwidth]{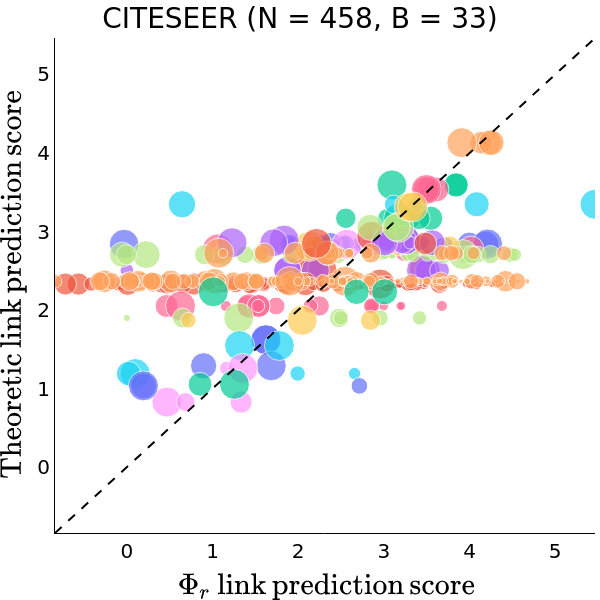}
    \end{subfigure}
    \begin{subfigure}[t]{0.4\textwidth}
        \centering
        \includegraphics[width=\textwidth]{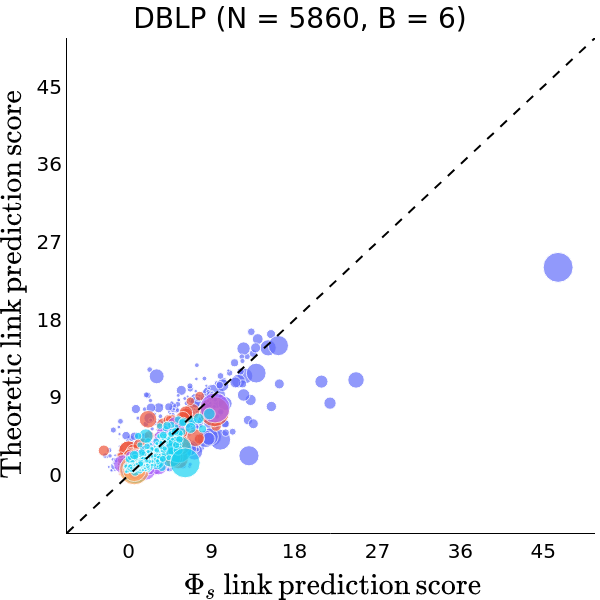}
    \end{subfigure}%
    \begin{subfigure}[t]{0.4\textwidth}
        \centering
        \includegraphics[width=\textwidth]{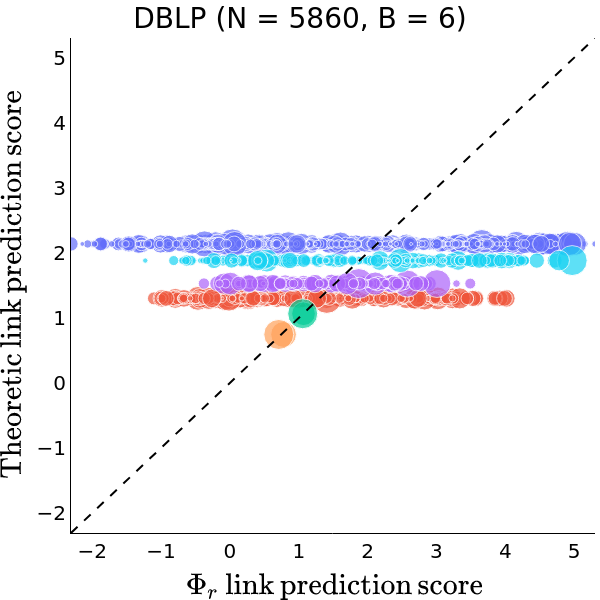}
    \end{subfigure}
    \begin{subfigure}[t]{0.4\textwidth}
        \centering
        \includegraphics[width=\textwidth]{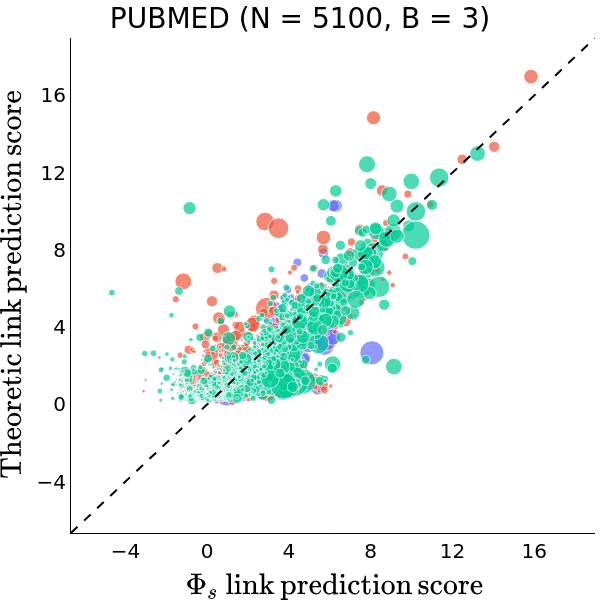}
    \end{subfigure}%
    \begin{subfigure}[t]{0.4\textwidth}
        \centering
        \includegraphics[width=\textwidth]{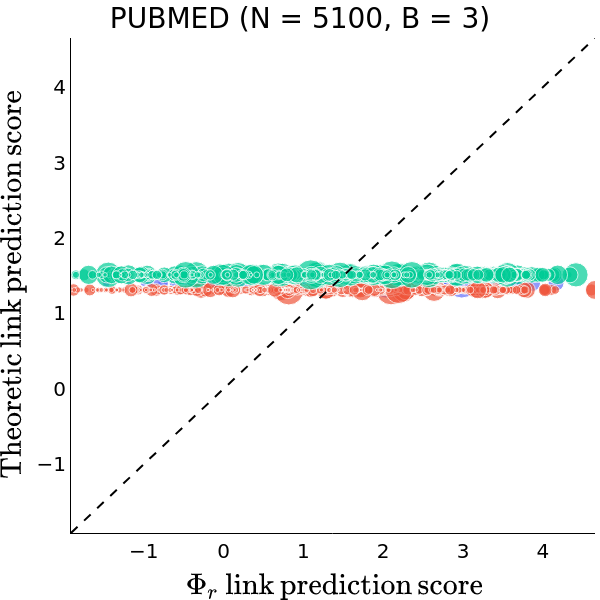}
    \end{subfigure}
    \caption{Theoretic vs. GCN LP scores for citation network datasets.}
    \label{fig:citation}
\end{figure*}

\begin{figure*}[t!]
    \centering
    \begin{subfigure}[t]{0.4\textwidth}
        \centering
        \includegraphics[width=\textwidth]{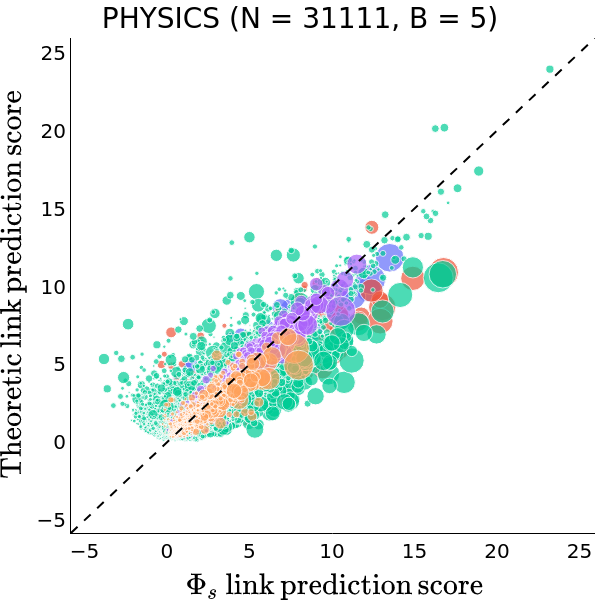}
    \end{subfigure}%
    \begin{subfigure}[t]{0.4\textwidth}
        \centering
        \includegraphics[width=\textwidth]{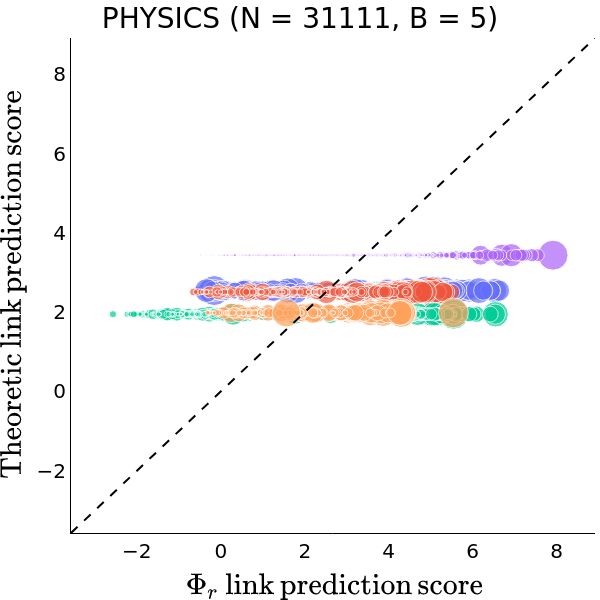}
    \end{subfigure}
    \caption{Theoretic vs. GCN LP scores for collaboration network datasets.}
    \label{fig:collaboration}
\end{figure*}

\begin{figure*}[t!]
    \centering
    \begin{subfigure}[t]{0.4\textwidth}
        \centering
        \includegraphics[width=\textwidth]{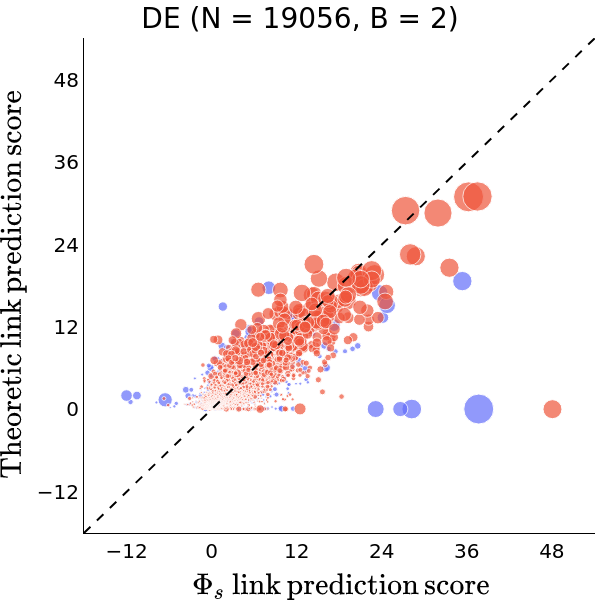}
    \end{subfigure}%
    \begin{subfigure}[t]{0.4\textwidth}
        \centering
        \includegraphics[width=\textwidth]{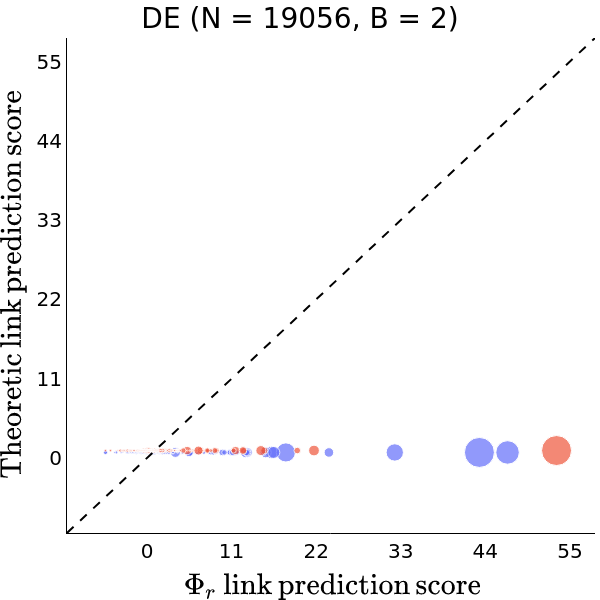}
    \end{subfigure}
    \begin{subfigure}[t]{0.4\textwidth}
        \centering
        \includegraphics[width=\textwidth]{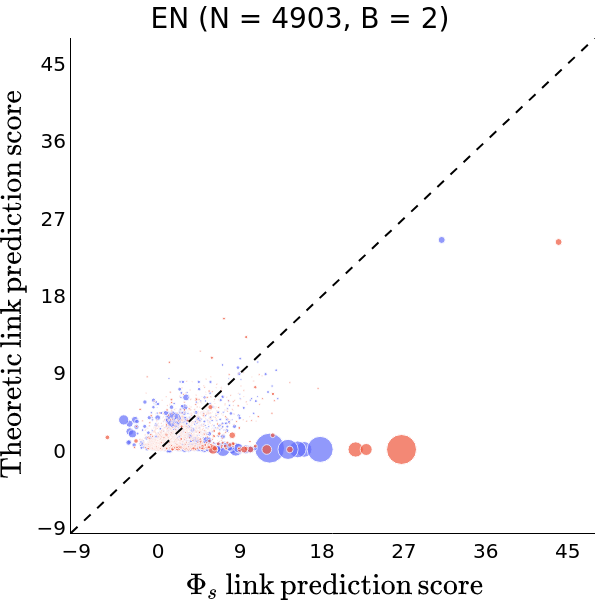}
    \end{subfigure}%
    \begin{subfigure}[t]{0.4\textwidth}
        \centering
        \includegraphics[width=\textwidth]{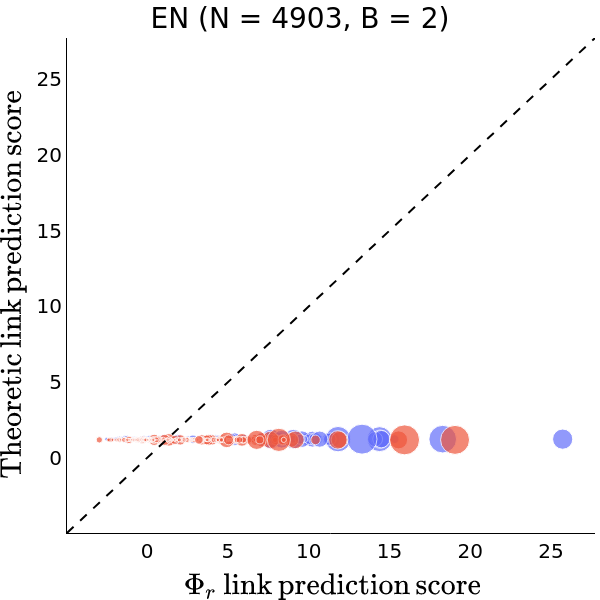}
    \end{subfigure}
    \begin{subfigure}[t]{0.4\textwidth}
        \centering
        \includegraphics[width=\textwidth]{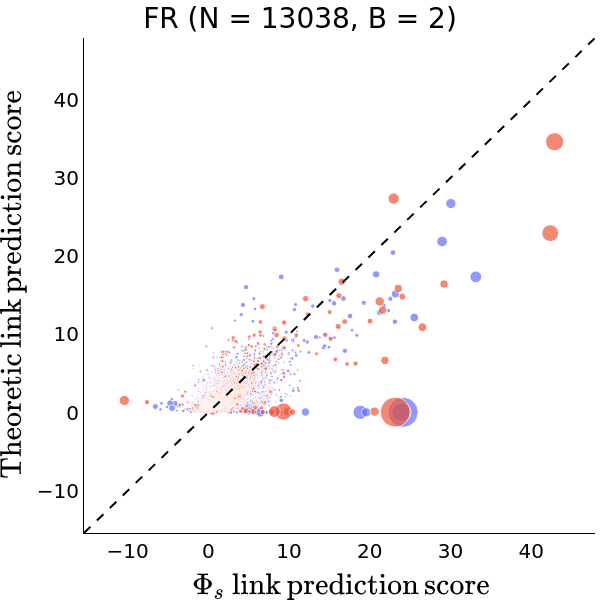}
    \end{subfigure}%
    \begin{subfigure}[t]{0.4\textwidth}
        \centering
        \includegraphics[width=\textwidth]{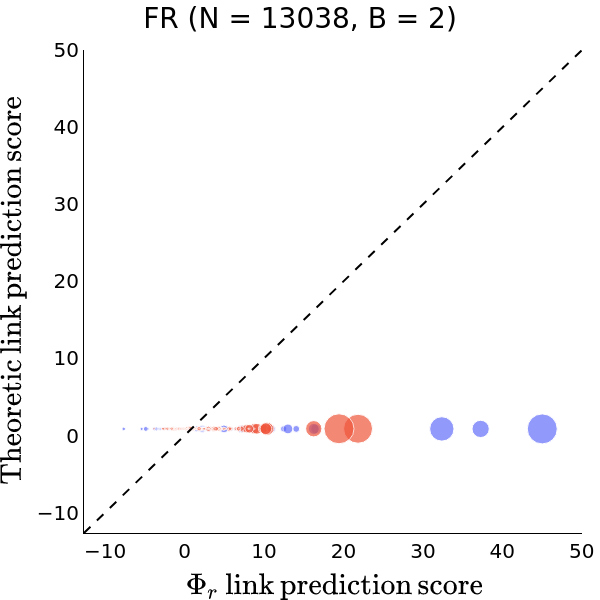}
    \end{subfigure}
    \caption{Theoretic vs. GCN LP scores for online social network datasets.}
    \label{fig:online}
\end{figure*}

\clearpage

\section{Additional Experiments}
\label{sec:additional-experiments}

\subsection{Additional Experiments for \S\ref{sec:validating-theory} (4-layer Encoders)}

We run the experiments from \S\ref{sec:validating-theory} for $\Phi_s$ with the same settings, except we use 4-layer (instead of 2-layer) encoders (128-dimensional hidden layers, 64-dimensional output layer). We run these additional experiments because the error bound for the theoretic LP scores for $\Phi_s$ depends on the number of encoder layers $L$. We find that the experimental results continue to support our theoretical analysis, both qualitatively and quantitatively (cf. Table \ref{tab:additional-exp}, Figure \ref{fig:additional-exp}); the NRMSE and PCC values are comparable to or better than those from the experiments with the 2-layer encoders (especially for the EN dataset).

\begin{table}[!ht]
\caption{The test AUC of the 4-layer $\Phi_s$ encoders on the real-world network datasets, and the NRMSE and PCC of the theoretic LP scores as predictors of the $\Phi_s$ scores.}
\label{tab:additional-exp}
\centering
\begin{tabular}{lrrr}
\toprule
            & \textbf{NRMSE} ($\downarrow$) & \textbf{PCC} ($\uparrow$) & $\Phi_s$ \textbf{Test AUC} ($\uparrow$) \\
\midrule
       CORA &             $0.044 \pm 0.006$ &         $0.858 \pm 0.026$ &              $0.853 \pm 0.028$ \\
   CITESEER &             $0.057 \pm 0.006$ &         $0.890 \pm 0.017$ &              $0.861 \pm 0.026$ \\
       DBLP &             $0.021 \pm 0.002$ &         $0.885 \pm 0.054$ &              $0.887 \pm 0.019$ \\
     PUBMED &             $0.056 \pm 0.009$ &         $0.802 \pm 0.024$ &              $0.900 \pm 0.006$ \\
         CS &             $0.039 \pm 0.006$ &         $0.918 \pm 0.008$ &              $0.949 \pm 0.004$ \\
    PHYSICS &             $0.030 \pm 0.002$ &         $0.077 \pm 0.013$ &              $0.950 \pm 0.004$ \\
 LASTFMASIA &             $0.040 \pm 0.004$ &         $0.938 \pm 0.005$ &              $0.949 \pm 0.002$ \\
         DE &             $0.014 \pm 0.003$ &         $0.918 \pm 0.025$ &              $0.882 \pm 0.002$ \\
         EN &             $0.034 \pm 0.005$ &         $0.752 \pm 0.036$ &              $0.846 \pm 0.008$ \\
         FR &             $0.019 \pm 0.003$ &         $0.833 \pm 0.038$ &              $0.896 \pm 0.003$ \\
\bottomrule
\end{tabular}
\end{table}

\begin{figure*}[t!]
    \centering
    \begin{subfigure}[t]{0.33\textwidth}
        \centering
        \includegraphics[width=\textwidth]{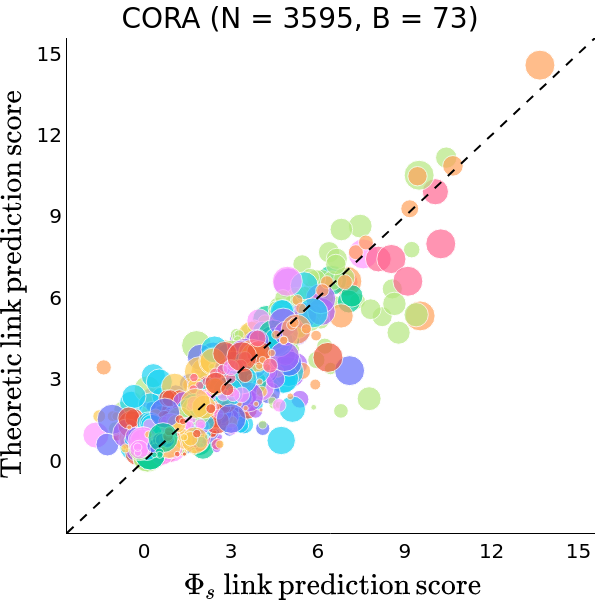}
    \end{subfigure}%
    \begin{subfigure}[t]{0.33\textwidth}
        \centering
        \includegraphics[width=\textwidth]{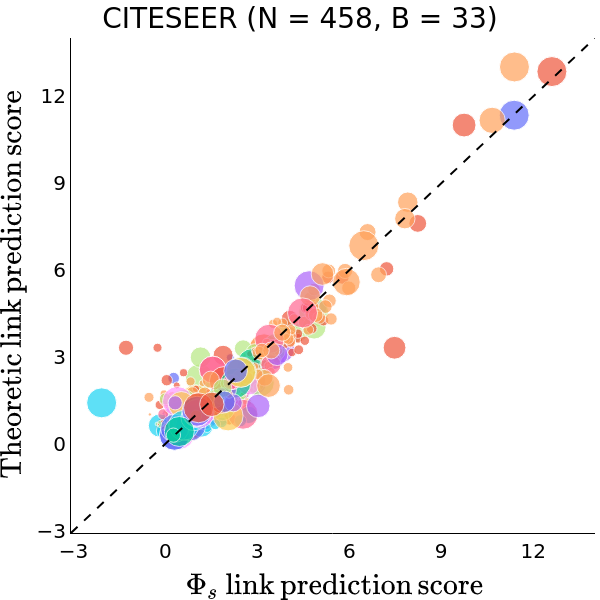}
    \end{subfigure}%
    \begin{subfigure}[t]{0.33\textwidth}
        \centering
        \includegraphics[width=\textwidth]{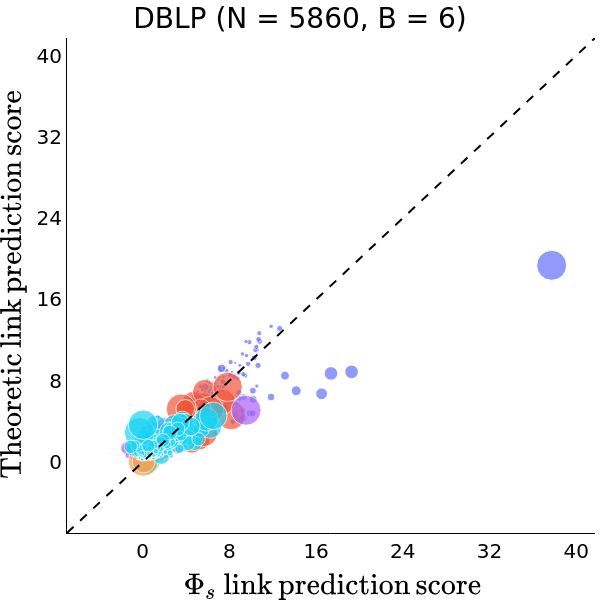}
    \end{subfigure}
    \begin{subfigure}[t]{0.33\textwidth}
        \centering
        \includegraphics[width=\textwidth]{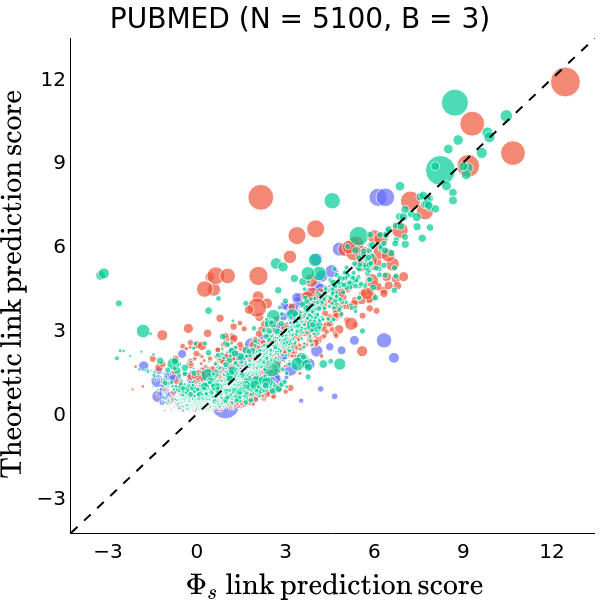}
    \end{subfigure}%
    \begin{subfigure}[t]{0.33\textwidth}
        \centering
        \includegraphics[width=\textwidth]{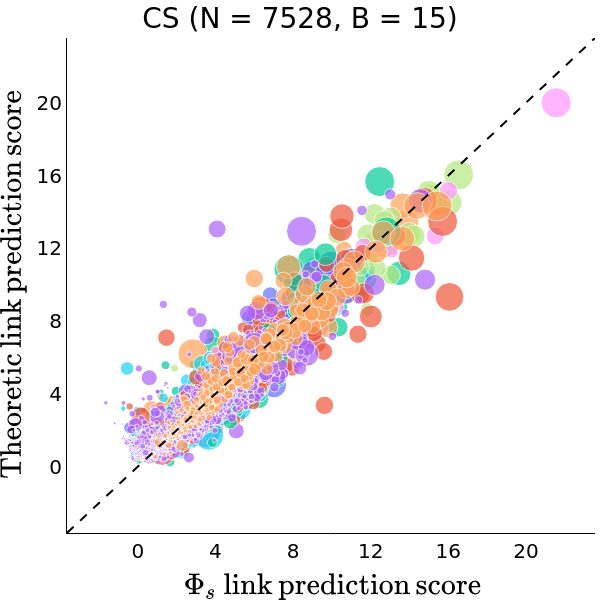}
    \end{subfigure}%
    \begin{subfigure}[t]{0.33\textwidth}
        \centering
        \includegraphics[width=\textwidth]{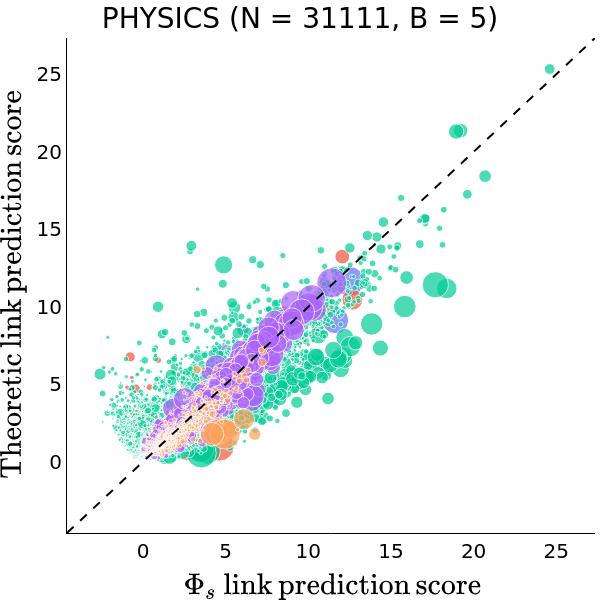}
    \end{subfigure}
    \begin{subfigure}[t]{0.33\textwidth}
        \centering
        \includegraphics[width=\textwidth]{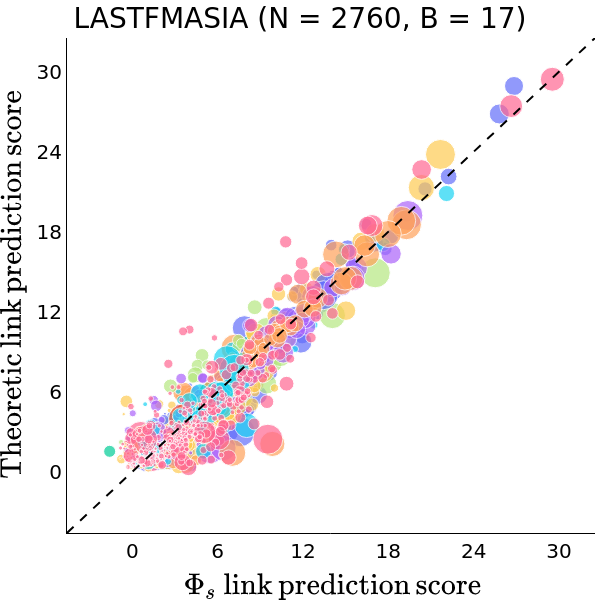}
    \end{subfigure}%
    \begin{subfigure}[t]{0.33\textwidth}
        \centering
        \includegraphics[width=\textwidth]{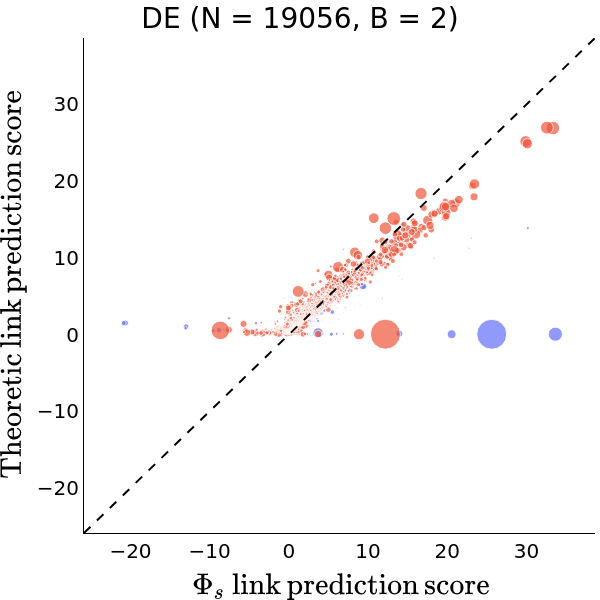}
    \end{subfigure}%
    \begin{subfigure}[t]{0.33\textwidth}
        \centering
        \includegraphics[width=\textwidth]{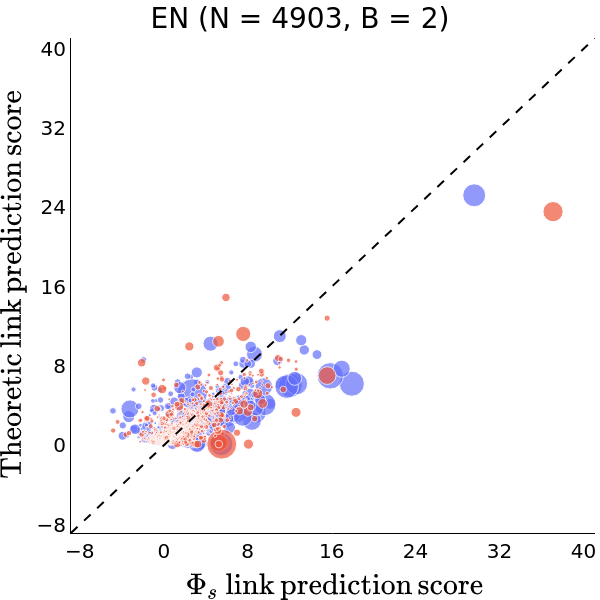}
    \end{subfigure}
    \begin{subfigure}[t]{0.33\textwidth}
        \centering
        \includegraphics[width=\textwidth]{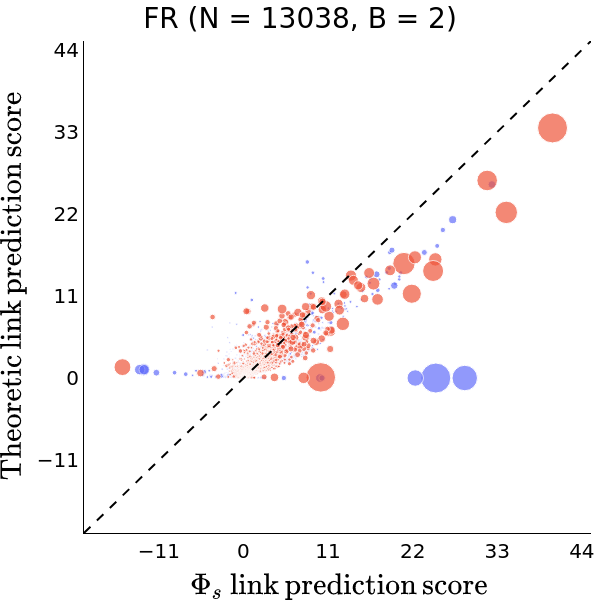}
    \end{subfigure}
   
    \caption{Theoretic LP score vs. 4-layer $\Phi_s$ LP score for all network datasets.}
    \label{fig:additional-exp}
\end{figure*}
\FloatBarrier

\clearpage

\subsection{Additional Experiments for \S\ref{sec:validating-theory} (Hadamard Product and MLP LP Score Function)}
\label{sec:hadamard-exp}

\add{We also run the experiments from \S\ref{sec:validating-theory} for $\Phi_s$ with the same settings, except we use the following LP score function:}
\begin{align}
    f_{LP} \left(\vh_i^{(L)}, \vh_j^{(L)}\right) = f_{MLP} \left( \vh_i^{(L)} \odot \vh_j^{(L)} \right),
\end{align}
\add{where $\odot$ is the Hadamard product and $f_{MLP}$ is a 2-layer MLP with a 64-dimensional hidden layer and ReLU nonlinearity. We run these additional experiments because a Hadamard product and MLP score function is often used in the literature. We find that that our theoretical analysis is still relevant to and reasonably supports the experimental results, both qualitatively and quantitatively (cf. Table \ref{tab:additional-exp-mlp}, Figure \ref{fig:additional-exp-mlp}). This could be because MLPs have an inductive bias towards learning simpler, often linear functions \citep{nakkiran2019sgd, vallepérez2019deep}, and our theoretical findings are generalizable to linear LP score functions. Notably, in this setting, $\Phi_s$ makes a higher number of negative link predictions. For a few datasets (e.g., Cora, CiteSeer, LastFMAsia), a handful of theoretic LP scores are negative because the regression (incorrectly) predicts $\overline{\rho}_s^2 (b)$ for 1-2 groups $S^{(b)}$ to be negative.}

\begin{table}[!ht]
\caption{\add{The test AUC of the $\Phi_s$ encoders with an $f_{MLP}$ score function on the real-world network datasets, and the NRMSE and PCC of the theoretic LP scores as predictors of the $\Phi_s$ scores.}}
\label{tab:additional-exp-mlp}
\centering
\begin{tabular}{lrrr}
\toprule
& \textbf{NRMSE} ($\downarrow$) & \textbf{PCC} ($\uparrow$) & $\Phi_s$ \textbf{Test AUC} ($\uparrow$) \\
\midrule
CORA & $0.034 \pm 0.004$ & $0.830 \pm 0.015$ & $0.915 \pm 0.001$ \\
CITESEER & $0.090 \pm 0.014$ & $0.365 \pm 0.070$ & $0.913 \pm 0.008$ \\
DBLP & $0.026 \pm 0.003$ & $0.652 \pm 0.029$ & $0.933 \pm 0.004$ \\
PUBMED & $0.054 \pm 0.007$ & $0.813 \pm 0.038$ & $0.932 \pm 0.003$ \\
CS & $0.047 \pm 0.008$ & $0.677 \pm 0.036$ & $0.970 \pm 0.001$ \\
PHYSICS & $0.055 \pm 0.007$ & $0.566 \pm 0.026$ & $0.976 \pm 0.001$ \\
LASTFMASIA & $0.049 \pm 0.008$ & $0.682 \pm 0.035$ & $0.960 \pm 0.003$ \\
DE & $0.030 \pm 0.008$ & $0.683 \pm 0.047$ & $0.935 \pm 0.001$ \\
EN & $0.039 \pm 0.006$ & $0.463 \pm 0.022$ & $0.905 \pm 0.002$ \\
FR & $0.031 \pm 0.006$ & $0.654 \pm 0.067$ & $0.935 \pm 0.002$ \\
\bottomrule
\end{tabular}
\end{table}

\begin{figure*}[t!]
    \centering
    \begin{subfigure}[t]{0.33\textwidth}
        \centering
        \includegraphics[width=\textwidth]{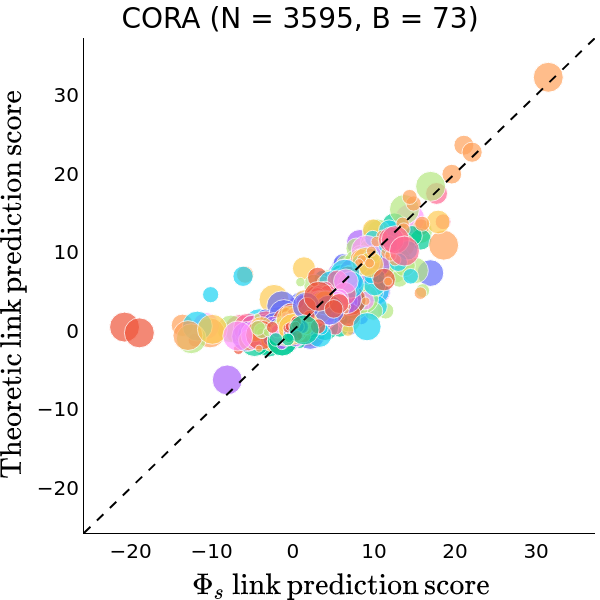}
    \end{subfigure}%
    \begin{subfigure}[t]{0.33\textwidth}
        \centering
        \includegraphics[width=\textwidth]{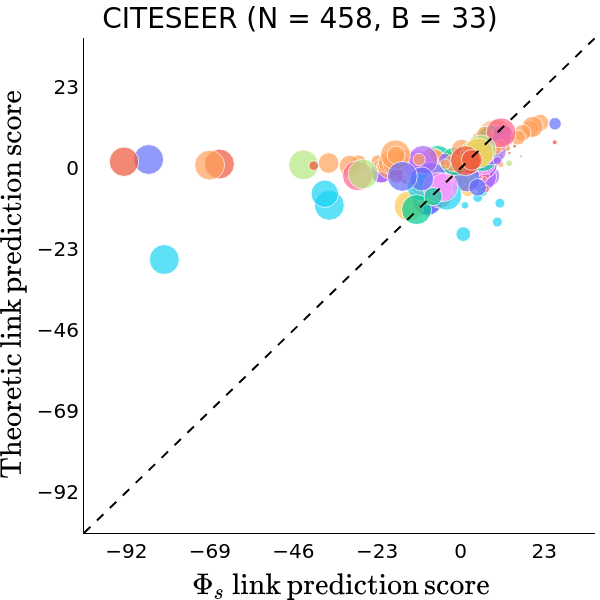}
    \end{subfigure}%
    \begin{subfigure}[t]{0.33\textwidth}
        \centering
        \includegraphics[width=\textwidth]{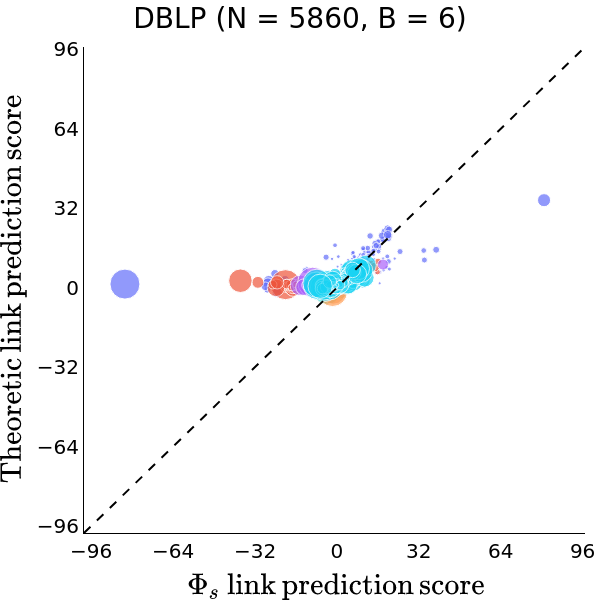}
    \end{subfigure}
    \begin{subfigure}[t]{0.33\textwidth}
        \centering
        \includegraphics[width=\textwidth]{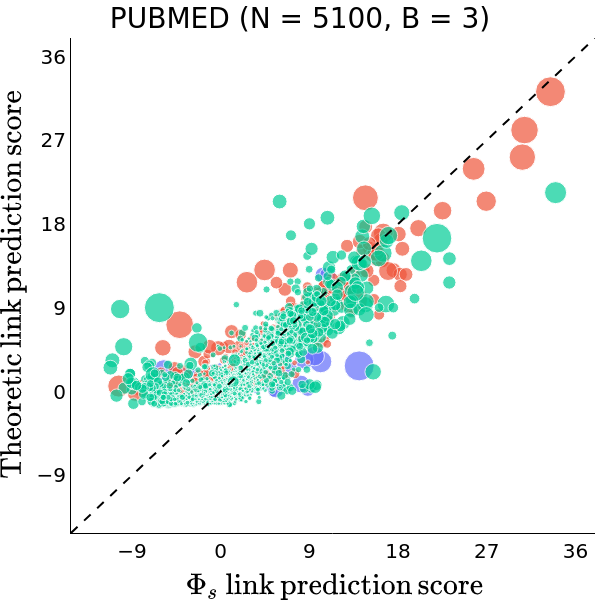}
    \end{subfigure}%
    \begin{subfigure}[t]{0.33\textwidth}
        \centering
        \includegraphics[width=\textwidth]{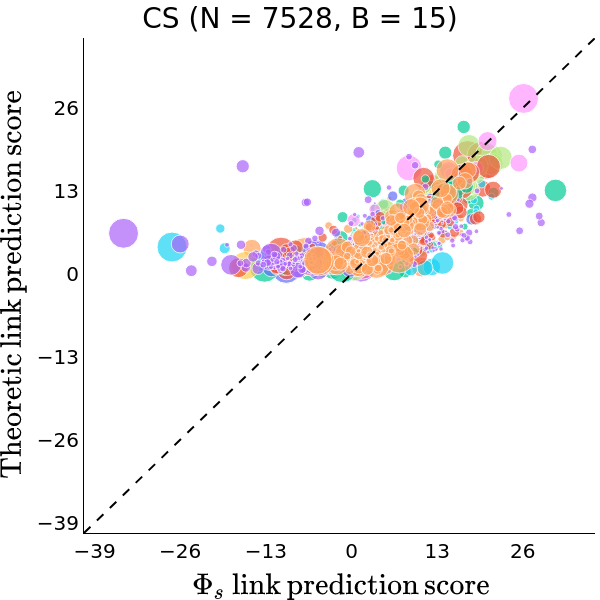}
    \end{subfigure}%
    \begin{subfigure}[t]{0.33\textwidth}
        \centering
        \includegraphics[width=\textwidth]{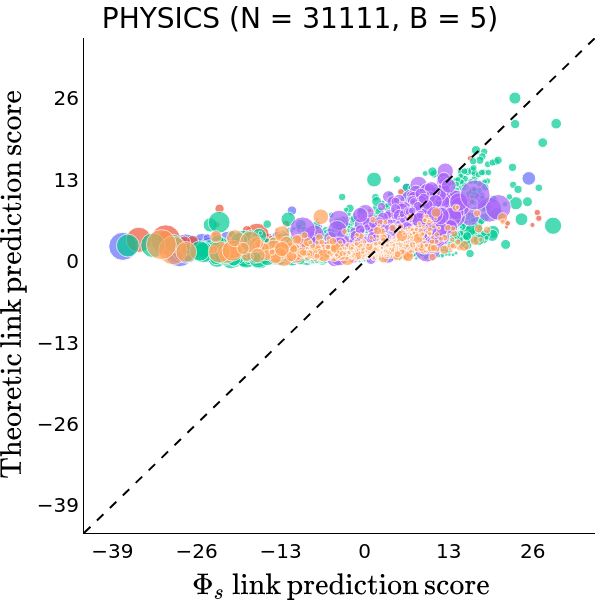}
    \end{subfigure}
    \begin{subfigure}[t]{0.33\textwidth}
        \centering
        \includegraphics[width=\textwidth]{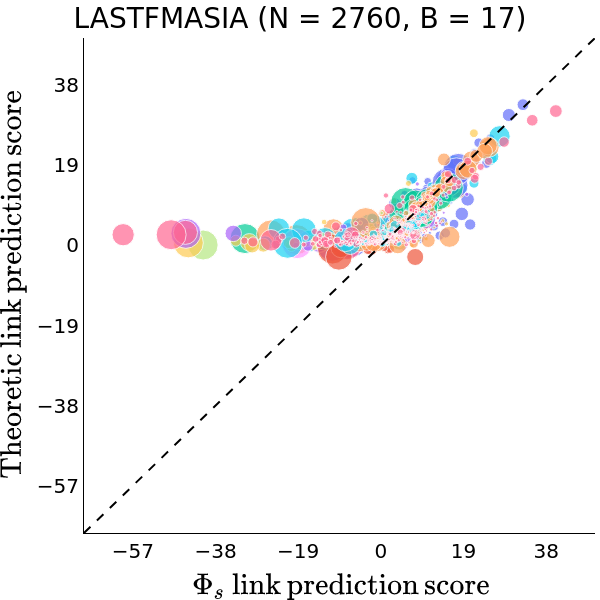}
    \end{subfigure}%
    \begin{subfigure}[t]{0.33\textwidth}
        \centering
        \includegraphics[width=\textwidth]{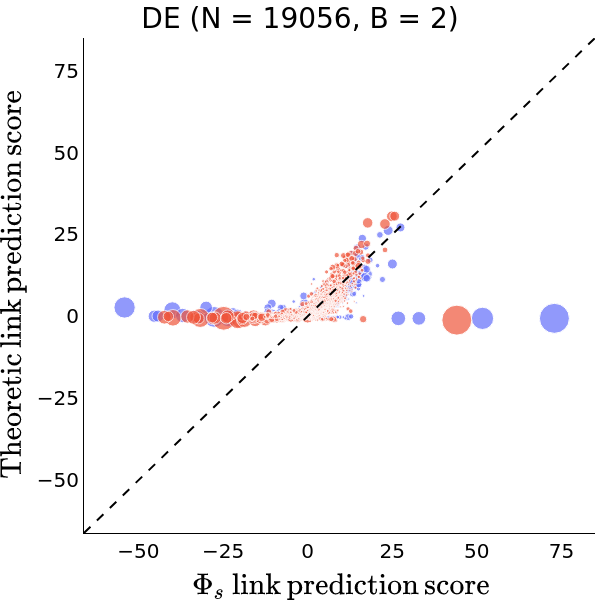}
    \end{subfigure}%
    \begin{subfigure}[t]{0.33\textwidth}
        \centering
        \includegraphics[width=\textwidth]{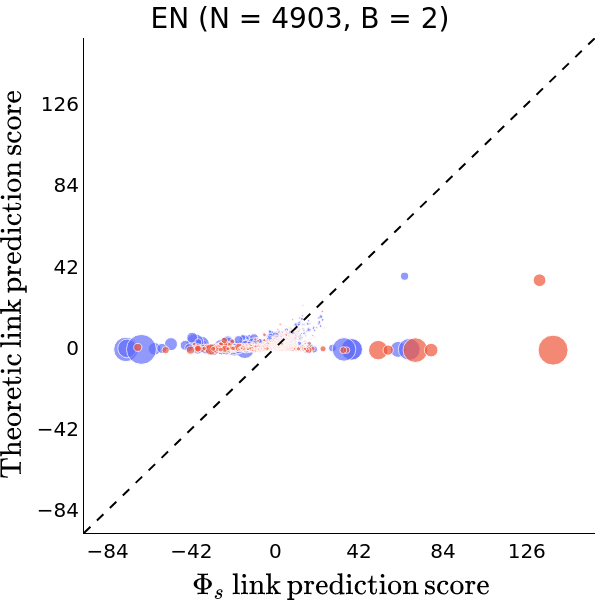}
    \end{subfigure}
    \begin{subfigure}[t]{0.33\textwidth}
        \centering
        \includegraphics[width=\textwidth]{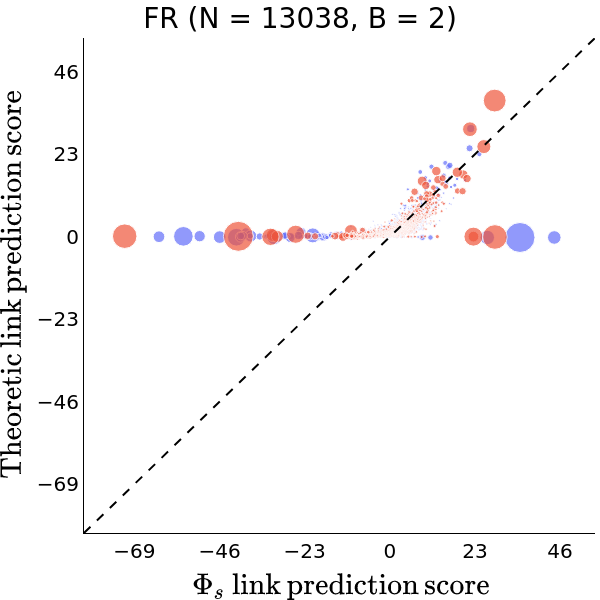}
    \end{subfigure}
   
    \caption{\add{Theoretic LP score vs. $\Phi_s$ LP score (with Hadamard product and MLP) for all network datasets.}}
    \label{fig:additional-exp-mlp}
\end{figure*}
\FloatBarrier

\clearpage

\subsection{Additional Experiments for \S\ref{sec:fairness-implications-exp}}

\begin{figure*}[ht!]
    \centering
    \begin{subfigure}[t]{0.33\textwidth}
        \centering
        \includegraphics[width=\textwidth]{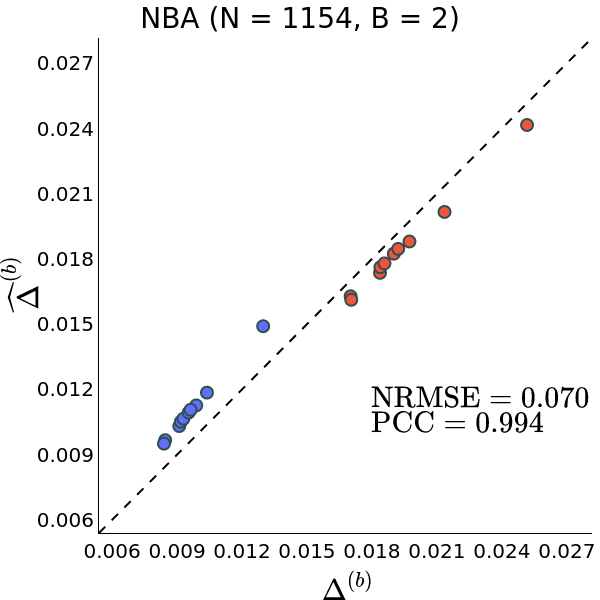}
    \end{subfigure}%
    \begin{subfigure}[t]{0.33\textwidth}
        \centering
        \includegraphics[width=\textwidth]{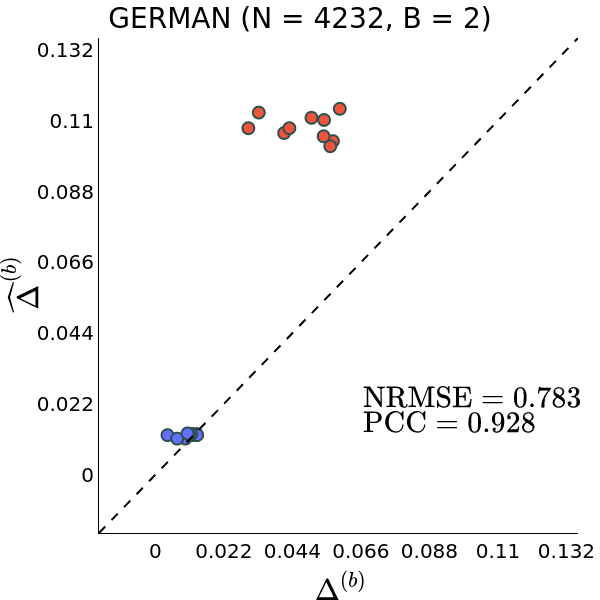}
    \end{subfigure}%
    \begin{subfigure}[t]{0.33\textwidth}
        \centering
        \includegraphics[width=\textwidth]{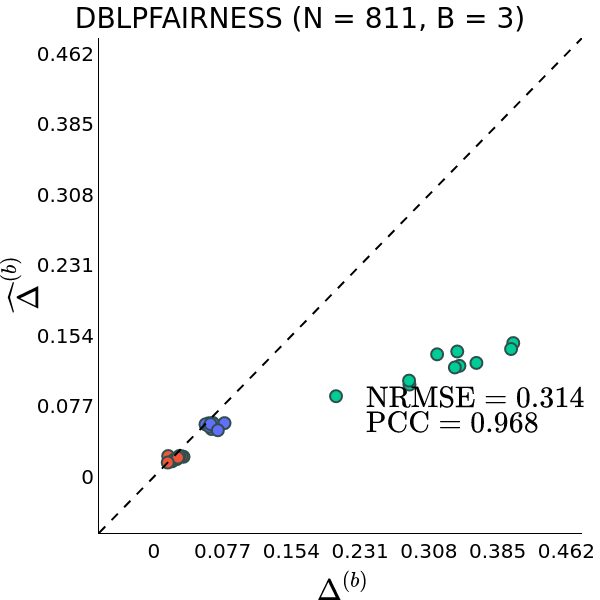}
    \end{subfigure}%
    \caption{The plots display $\widehat{\Delta}^{(b)}$ vs. $\Delta^{(b)}$ for 4-layer $\Phi_s$ for the NBA, German, and DBLP-Fairness datasets over all $b \in [B]$ and 10 random seeds. 
    }
\end{figure*}
\FloatBarrier

\subsection{Additional Experiments for \S\ref{sec:fairness-regularizer}}

\begin{table*}[!ht]
\caption{$\frac{1}{B} \sum_{b \in [B]} \Delta^{(b)}$ and the test AUC for the NBA, German, and DBLP-Fairness datasets with various settings of $\lambda_{\text{fair}}$. The \textbf{left} table corresponds to 4-layer $\Phi_s$, and the \textbf{right} to 4-layer $\Phi_r$.}
\centering
\begin{adjustbox}{max width=0.48\textwidth}
\begin{tabular}{lrrr}
\toprule
              & $\lambda_{\text{fair}}$ & $\frac{1}{B} \sum_{b \in [B]} \Delta^{(b)}$ ($\downarrow$) & $\Phi_s$ \textbf{Test AUC} ($\uparrow$) \\
\midrule
          NBA &                     4.0 &                       $0.000 \pm 0.000$ &                       $0.752 \pm 0.001$ \\
          NBA &                     2.0 &                       $0.006 \pm 0.001$ &                       $0.752 \pm 0.001$ \\
          NBA &                     1.0 &                       $0.011 \pm 0.001$ &                       $0.753 \pm 0.001$ \\
          NBA &                     0.0 &                       $0.014 \pm 0.001$ &                       $0.753 \pm 0.001$ \\
\hline
 DBLPFAIRNESS &                     4.0 &                       $0.090 \pm 0.041$ &                       $0.793 \pm 0.009$ \\
 DBLPFAIRNESS &                     2.0 &                       $0.070 \pm 0.015$ &                       $0.800 \pm 0.007$ \\
 DBLPFAIRNESS &                     1.0 &                       $0.099 \pm 0.009$ &                       $0.804 \pm 0.007$ \\
 DBLPFAIRNESS &                     0.0 &                       $0.122 \pm 0.028$ &                       $0.820 \pm 0.009$ \\
 \hline
       GERMAN &                     4.0 &                       $0.012 \pm 0.008$ &                       $0.817 \pm 0.004$ \\
       GERMAN &                     2.0 &                       $0.018 \pm 0.007$ &                       $0.827 \pm 0.015$ \\
       GERMAN &                     1.0 &                       $0.018 \pm 0.008$ &                       $0.856 \pm 0.025$ \\
       GERMAN &                     0.0 &                       $0.028 \pm 0.007$ &                       $0.874 \pm 0.011$ \\
\bottomrule
\end{tabular}
\end{adjustbox}
\quad
\begin{adjustbox}{max width=0.48\textwidth}
\begin{tabular}{lrrr}
\toprule
              & $\lambda_{\text{fair}}$ & $\frac{1}{B} \sum_{b \in [B]} \Delta^{(b)}$ ($\downarrow$) & $\Phi_r$ \textbf{Test AUC} ($\uparrow$) \\
\midrule
NBA & 4.0 & $0.000 \pm 0.000$ & $0.581 \pm 0.029$ \\
NBA & 2.0 & $0.000 \pm 0.000$ & $0.574 \pm 0.021$ \\
NBA & 1.0 & $0.000 \pm 0.000$ & $0.580 \pm 0.025$ \\
NBA & 0.0 & $0.000 \pm 0.000$ & $0.589 \pm 0.031$ \\
\midrule
DBLPFAIRNESS & 4.0 & $0.034 \pm 0.012$ & $0.769 \pm 0.009$ \\
DBLPFAIRNESS & 2.0 & $0.045 \pm 0.021$ & $0.788 \pm 0.007$ \\
DBLPFAIRNESS & 1.0 & $0.074 \pm 0.013$ & $0.797 \pm 0.006$ \\
DBLPFAIRNESS & 0.0 & $0.095 \pm 0.015$ & $0.811 \pm 0.006$ \\
\midrule
GERMAN & 4.0 & $0.027 \pm 0.009$ & $0.765 \pm 0.013$ \\
GERMAN & 2.0 & $0.023 \pm 0.007$ & $0.765 \pm 0.011$ \\
GERMAN & 1.0 & $0.031 \pm 0.010$ & $0.786 \pm 0.030$ \\
GERMAN & 0.0 & $0.030 \pm 0.009$ & $0.838 \pm 0.025$ \\
\bottomrule
\end{tabular}
\end{adjustbox}
\end{table*}

\clearpage

\section{Theory Pitfalls}
\label{sec:theory-pitfalls}

To understand the second pitfall from \S\ref{sec:validating-theory}, we separately investigate the association between the within-group degree product $\left( \widehat{\mD}_{i i} \widehat{\mD}_{j j} \right)$ and the absolute deviation of the theoretic LP scores from the $\Phi_s$ scores, as well as the association between the (transformed) feature similarity $\left( \left\| \sum_{k \in S^{(b)}} \frac{ \sqrt{\widehat{\mD}_{k k}}}{\text{vol} ({\cal G}^{(b)})} \alpha_k \right\|^2_2 \right)$ and the absolute deviation (cf. Figure \ref{fig:pitfalls-citeseer-feat-deg}). We observe that the absolute deviation is highest for the node pairs with a relatively small degree product (i.e., nodes with a low PA score) and low feature similarity.

\begin{figure*}[!ht]
\centering
   \begin{subfigure}[t]{0.48\textwidth}
        \centering
        \includegraphics[width=\textwidth]{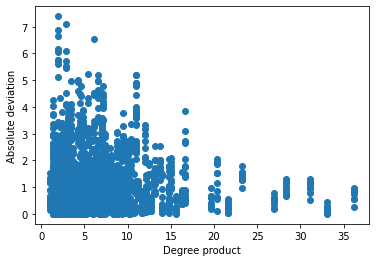}
    \end{subfigure}%
    \begin{subfigure}[t]{0.48\textwidth}
        \centering
        \includegraphics[width=\textwidth]{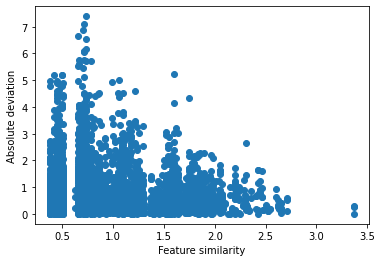}
    \end{subfigure}
   
    \caption{Associations of absolute deviation with degree product and with feature similarity for CiteSeer.}
    \label{fig:pitfalls-citeseer-feat-deg}
\end{figure*}
\FloatBarrier

\clearpage

\section{Error Analysis of $\Phi_r$ Theoretic Scores}
\label{sec:rw-error-analysis}

Figure \ref{fig:rw-error} reveals that the max term $\max_{u, v \in {\cal V}} \sqrt{\frac{\widehat{\mD}_{v v}}{\widehat{\mD}_{u u}}}$ is quite large in practice, which causes the theoretic LP scores to generally be poor estimates for the $\Phi_r$ scores.
\add{We additionally find in Figure \ref{fig:rw-error} that the relative error (as measured by NRMSE and PCC) of the theoretic LP scores for $\Phi_r$ is not lower for lower values of the max term $\max_{u, v \in {\cal V}} \sqrt{\frac{\widehat{\mD}_{v v}}{\widehat{\mD}_{u u}}}$.}

\begin{figure*}[!ht]
\centering
   \begin{subfigure}[t]{0.48\textwidth}
        \centering
        \includegraphics[width=\textwidth]{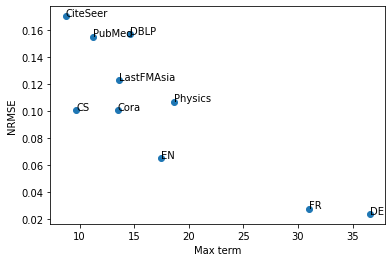}
    \end{subfigure}%
    \begin{subfigure}[t]{0.48\textwidth}
        \centering
        \includegraphics[width=\textwidth]{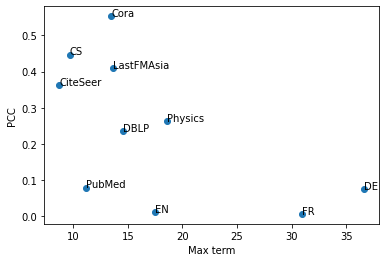}
    \end{subfigure}
   
    \caption{\add{Weak associations of max term with NRMSE and PCC of theoretic LP scores for $\Phi_r$ across all datasets described in \S\ref{sec:validating-datasets}.}}
    \label{fig:rw-error}
\end{figure*}
\FloatBarrier

\add{Furthermore, Figure \ref{fig:rw-deg-citeseer} reveals that $\Phi_r$ LP scores are \textit{not} higher for incident nodes with larger degrees.}

\begin{figure*}[!ht]
\centering
   \begin{subfigure}[t]{0.48\textwidth}
        \centering
        \includegraphics[width=\textwidth]{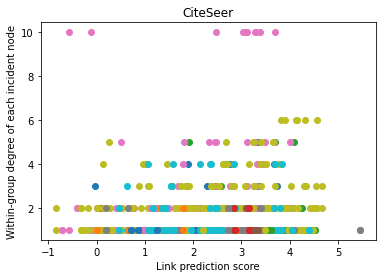}
    \end{subfigure}%
    \begin{subfigure}[t]{0.48\textwidth}
        \centering
        \includegraphics[width=\textwidth]{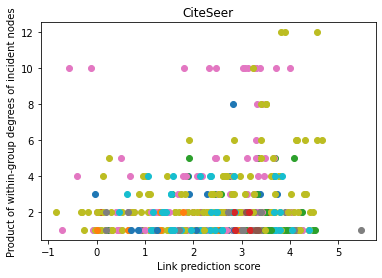}
    \end{subfigure}
   
    \caption{\add{Weak associations of mean $\Phi_r$ LP scores (over 10 random seeds) with degree of each incident node and product of degrees of both incident nodes. Colors correspond to different groups.}}
    \label{fig:rw-deg-citeseer}
\end{figure*}
\FloatBarrier

There are intimate connections between Theorem \ref{thm:rw} and the steady-state probabilities of random walks. The stationary probabilities of random walks are the same regardless of the starting node. This is why $\Phi_r$ produces similar representations for all the nodes in each social group, regardless of the degree of the node; in fact, with a larger number of layers, $\Phi_r$ would oversmooth all the representations to the same vector \citep{Keriven2022NotTL}. Hence, $\Phi_r$ LP scores do not have a degree dependence, theoretically or empirically.

\clearpage

\section{Preferential Attachment and Motivation}
\label{sec:pa-motivation}

\paragraph{Preferential Attachment}
Preferential attachment (PA) describes the propensity of links to form with high-degree nodes\footnote{\url{https://networkx.org/documentation/stable/reference/algorithms/generated/networkx.algorithms.link_prediction.preferential_attachment.html}}. Network scientists have studied for decades how links in real-world networks exhibit PA. For example, in the iterative Barabási-Albert model of network formation, each new node $s$ forms links with existing nodes $t$ with probability proportional to the degree of $t$, i.e., $\mathbb{P} ((s, t) \in {\cal E}) \propto deg(t)$. In the context of our paper, PA describes how a GCN with an inner-product LP score function often predicts links between nodes $i, j$ with score $\propto \sqrt{deg(i) \cdot deg(j)}$ approximately (Theorem \ref{thm:sym}).

\paragraph{Motivation} A wealth of literature in network science and the social sciences has examined the PA properties of real-world networks and how these properties contribute to unfair (non-neural) algorithms (\S\ref{sec:related-work}). For example, \citet{stoica2018ceiling} find that Instagram accounts run by men have a significantly higher following than those run by women due to gender discrimination; this degree disparity is only amplified by link recommendation algorithms that suggest following high-degree accounts, which makes the rich get richer and reveals that these algorithms have a PA bias. Moreover, many papers outside graph learning have discussed the intersectional unfairness of machine learning (\S\ref{sec:related-work}).

However, despite the increasing real-world deployment of GNNs for LP, their unfairness has not been studied from the perspectives of PA and intersections of social groups. Our paper fills this gap by providing thorough theoretical and empirical evidence that GCNs \citep{kipf2017semisupervised} have a PA bias when predicting links between nodes in the same social group. \textbf{This finding is nontrivial as GCNs leverage a combination of features and local structural context to make link predictions.}

Our research question is challenging from a technical perspective, as it requires uncovering properties of {\em short} random walks on graphs (since most GNNs are shallow); in contrast, most random walk results in the literature concern random walks at convergence. Our research question is further important because GNNs with a PA bias can amplify degree disparities, which translates to increased discrimination and disparities in social influence among nodes. 

As we uncover this new form of unfairness, there are no existing solutions to this unfairness in the literature. We propose a training-time regularization-based fairness method that alleviates this unfairness without greatly sacrificing the test AUC of LP. While capping the number of positive link predictions per node is a possible solution, doing so with utility in mind requires identifying a utility-maximizing subset of link predictions. As our theoretical and empirical results reveal, GCN LP scores are often inherently proportional to the geometric mean of the degrees of the incident nodes, which can make them a poor indicator of prediction confidence; from a calibration perspective, GCNs naturally make overconfident predictions for links between high-degree nodes.

While we describe methods for alleviating degree bias in \S\ref{sec:related-work}, these methods address degraded performance for low-degree nodes, not PA bias. We do not study performance issues but rather how GCNs scale representations of nodes proportionally to (approximately) the square root of their within-group degree, which affects the magnitude of their LP scores (cf. \S\ref{sec:deg-bias-comments}). 

In summary, we augment the field’s understanding of degree bias beyond performance disparities across nodes. We further lay a foundation to study PA bias and within-group unfairness in GNN LP more broadly (e.g., SOTA contrastive methods for LP), which is a critical and interesting direction of research.

\clearpage

\section{Comparison to Prior Research on Degree Bias}
\label{sec:deg-bias-comments}

Studies concerning degree bias have observed that low-degree nodes experience degraded performance compared to high-degree nodes. They have thus often formulated degree bias from a performance perspective, focusing on equal opportunity. In particular, these studies seek to satisfy $\mathbb{P}(\hat{y}_v = y | y_v = y, deg(v) = d) = \mathbb{P}(\hat{y}_v = y | y_v = y, deg(v) = d')$ for all possible degrees $d, d’$, where $\hat{y}_v$ is the prediction for node $v$ and $y_v$ is its ground-truth label. This fairness criterion treats the degree of a node as a sensitive attribute, requiring that a GNN’s accuracy is consistent across nodes with different degrees.

However, in this paper, we seek to ensure that degree disparities in networks are not amplified by GNN LP. We cannot adopt the equal opportunity formulation of degree bias because it is concerned with performance while we are concerned with degree disparity amplification. For example, even if we consistently predict links with the same accuracy across nodes with different degrees, high-degree nodes can still receive higher LP scores than low-degree nodes. In this way, the ``degree bias'' discussed by other studies is not compatible with our unfairness metric (Eqn. \ref{eqn:fairness-metric}). We also cannot simply adopt common LP fairness metrics like dyadic fairness, as they do not capture the new type of unfairness that we uncover.

Roughly, we care that $\mathbb{E} [ \hat{y}_{u v} | deg(u) = d ] = \mathbb{E} [ \hat{y}_{u v} | deg(u) = d']$, where $\hat{y}_{u v}$ is the GNN score for a link prediction between nodes $u, v$. In other words, we do not want GNN LP scores to be higher for high-degree nodes vs. low-degree nodes. This is what motivates our fairness metric (Eqn. \ref{eqn:fairness-metric}).

Our theoretical analysis (Theorem \ref{thm:sym}) and empirical validation (\S\ref{sec:validating-theory}) reveal that GCNs fundamentally often predict links between nodes $i, j$ with score approximately $\propto \sqrt{deg(i) \cdot deg(j)}$ {\em because} of their symmetric normalized filter. This finding of a preferential attachment bias allows us to express our unfairness metric in terms of degree disparity (Eqn. \ref{eqn:approx-fairness-metric-sym}), but this degree disparity is {\em not} related to the ``degree bias'' that has been discussed by other papers; this is a new fairness paradigm.

\clearpage

\section{Justification of Assumptions in Lemma \ref{lemma:taylor}}
\label{sec:taylor-lemma-comments}

The independence of path activation probabilities may not always hold true in practice. However, we verify that this assumption is plausible via our extensive experiments on real-world datasets that validate our theoretical analysis (\S\ref{sec:validating-theory}). This assumption also aligns with findings that deep neural networks have an inductive bias towards learning simpler, often linear, functions \citep{nakkiran2019sgd, vallepérez2019deep}. Furthermore, a variant of our assumption (where $\rho(i) = \rho$ is constant for all nodes) has been used in the literature to simplify theoretical analysis (e.g., \citet{Xu2018RepresentationLO, tang2020degree}); our assumption may be more realistic than this variant, as it captures that the probability of paths activating can differ across nodes (e.g., due to differences in features, neighborhood structure).
\end{document}